\documentclass{article} 

\PassOptionsToPackage{sort&compress}{natbib}

\usepackage{iclr2027_conference,times}
\usepackage{todonotes}

\usepackage{amsmath,amsfonts,bm}

\def\figref#1{figure~\ref{#1}}

\def\secref#1{section~\ref{#1}}

\def\eqref#1{equation~\ref{#1}}

\def\plaineqref#1{\ref{#1}}

\def\algref#1{algorithm~\ref{#1}}

\def\1{\bm{1}}

\DeclareMathAlphabet{\mathsfit}{\encodingdefault}{\sfdefault}{m}{sl}
\SetMathAlphabet{\mathsfit}{bold}{\encodingdefault}{\sfdefault}{bx}{n}

\def\gX{{\mathcal{X}}}

\newcommand{\E}{\mathbb{E}}

\newcommand{\R}{\mathbb{R}}

\newcommand{\softmax}{\mathrm{softmax}}

\newcommand{\KL}{D_{\mathrm{KL}}}

\DeclareMathOperator*{\argmax}{arg\,max}

\usepackage{xcolor}
\usepackage{listings}
\usepackage{amssymb}
\usepackage{amsthm}
\usepackage{booktabs}
\usepackage{algorithm}
\usepackage{algorithmic}
\usepackage{url}
\usepackage{rotating}
\usepackage{subcaption}
\usepackage{xcolor}
\usepackage{colortbl}
\usepackage{graphicx}
\usepackage{enumitem}
\usepackage{wrapfig}
\usepackage{hyperref}   

\newif\ifshowrevisions
\DeclareRobustCommand{\rev}[1]{\ifshowrevisions\textcolor{blue}{#1}\else#1\fi}
\definecolor{codekw}{RGB}{0,0,160}
\definecolor{codecm}{RGB}{110,110,110}
\definecolor{codestr}{RGB}{0,110,0}
\definecolor{codeframe}{gray}{0.55}
\definecolor{codenum}{gray}{0.45}
\definecolor{codemark}{RGB}{170,0,0}

\lstdefinestyle{auditpy}{
  language=Python,
  basicstyle=\ttfamily\scriptsize,
  keywordstyle=\color{codekw}\bfseries,
  commentstyle=\color{codecm}\itshape,
  stringstyle=\color{codestr},
  numbers=left,
  numberstyle=\ttfamily\tiny\color{codenum},
  numbersep=7pt,
  frame=single,
  rulecolor=\color{codeframe},
  framesep=4pt,
  xleftmargin=2.4em,
  columns=fullflexible,
  keepspaces=true,
  breaklines=true,
  breakatwhitespace=true,
  breakindent=1.5em,
  postbreak=\mbox{\textcolor{codemark}{$\hookrightarrow$}\space},
  showstringspaces=false,
  captionpos=b,
  aboveskip=1.4ex,
  belowskip=0.8ex,
  escapeinside={(*@}{@*)},
}
\newtheorem{proposition}{Proposition}

\theoremstyle{remark}

\theoremstyle{definition}
\newtheorem{definition}{Definition}[section]

\renewcommand{\eqref}[1]{Eq.~(\ref{#1})}
\renewcommand{\secref}[1]{Section~\ref{#1}}
\renewcommand{\figref}[1]{Fig.~\ref{#1}}
\renewcommand{\algref}[1]{Algorithm~\ref{#1}}
\newcommand{\appref}[1]{Appendix~\ref{#1}}

\newcommand{\propref}[1]{Proposition~\ref{#1}}

\newcommand{\defref}[1]{Definition~\ref{#1}}
\newcommand{\tabref}[1]{Table~\ref{#1}}

\newcommand{\pstar}{p^{\star}}                
\newcommand{\Pstar}{P^{\star}}                
\newcommand{\PI}{\mathrm{PI}}
\newcommand{\EI}{\mathrm{EI}}
\newcommand{\acq}{\alpha}                     
\newcommand{\dataset}{\mathcal{D}}

\newcommand{\vocab}{\mathcal{V}}

\newcommand{\zhat}{\hat z}                         
\newcommand{\ztilt}{\hat z^{\mathrm{tilt}}}        
\newcommand{\ztiltg}{\hat z^{\mathrm{tilt},\gamma}}

\definecolor{cAFGGcdcd}{HTML}{E6194B}
\definecolor{cAFGGseq} {HTML}{FF7F0E}
\definecolor{cInvBO}   {HTML}{1F77B4}
\definecolor{cCoBO}    {HTML}{2CA02C}
\definecolor{cLOLBO}   {HTML}{9467BD}
\definecolor{cCOWBOYS} {HTML}{8C564B}

\definecolor{sigwin}{HTML}{D6F0D6}

\title{Acquisition Function-Guided Generation for Bayesian \rev{Optimization} over Discrete Spaces}

\author{Tommaso Ancilli$^{1,2*}$, Thomas Wedenig$^{1,2}$\thanks{Joint first author}\ , Henry Moss$^{3}$, Robert Peharz$^{1,2}$ \\
$^{1}$Institute of Machine Learning and Neural Computation, TU Graz, Austria \\
$^{2}$Graz Center for Machine Learning (GraML), Austria \\
$^{3}$School of Mathematical Sciences, Lancaster University, UK \\
\texttt{\{name.surname\}@tugraz.at}, \texttt{henry.moss@lancaster.ac.uk}
}

\iclrfinalcopy 

\begin{document}

\maketitle

\begin{abstract}
Bayesian optimization (BO) over discrete spaces is challenging due to missing gradients and the sparsity of valid solutions.
Latent-space BO sidesteps both issues by fitting a surrogate in the latent space of a generative model, which however substantially complicates the direct incorporation of domain knowledge that is naturally expressed in the original discrete space.
In this work, we instead combine a pretrained generative model with a surrogate defined directly in data space, where inductive biases such as domain-specific kernels can be readily incorporated.
Rather than maximizing the acquisition function over the discrete design space, our method samples candidates from the generative model under a distribution tilted by the acquisition function.
To recover useful gradients for sampling, we embed this tilted distribution into a diffusion bridge that yields a smooth, posterior-averaged acquisition function.
This leads to \emph{acquisition function guided generation} (AFGG), which we instantiate with a categorical diffusion model and an autoregressive normalizing flow.
Across a broad suite of molecular benchmarks, the diffusion model variant performs on par with leading BO methods, while the flow model variant outperforms state-of-the-art methods in most cases and remains highly competitive otherwise. Our implementation is available at \url{https://anonymous.4open.science/r/afgg}.


\end{abstract}


\section{Introduction}
\label{sec:intro}

Bayesian optimization (BO) offers a sample-efficient approach to problems such as molecular and protein design, where objective evaluations are costly and the search space is discrete and structured 
\citep{garnett2023bayesian, yang2025active}.
Unfortunately, direct optimization over such spaces is hard because valid designs are sparsely populated in a combinatorial space and standard continuous optimizers have no useful gradients to follow \citep{moss2020boss}.
Latent-space Bayesian optimization (LSBO) addresses both challenges by training a generative model over the discrete space---thereby capturing the sparse structure of valid designs---and by performing BO in its continuous latent space, where gradients for acquisition optimization are readily available \citep{gomez2018automatic}.

However, domain knowledge is harder to incorporate in an abstract latent space, while in data space, such knowledge can often be encoded directly through inductive biases such as domain-specific kernels \citep{rogers2010extended}.
Furthermore, the latent space must also ``do two jobs at once:'' support faithful generation and provide a geometry in which the objective is easy to model by a surrogate.
These goals need not align, which motivates many LSBO methods to adapt the latent representation during BO \citep{maus2022local, lee2023cobo, chu2024inversion, tripp2020sample, lee2025nfbo, deshwal2021}, despite the small-data regime in which such adaptation can overfit \citep{ober2021promisespitfallsdeepkernel}.
Moreover, even a strong generative model only approximates the sparse set of valid designs, so acquisition optimization can move into poorly supported regions of latent space where decoding becomes unreliable \citep{moss2025cowboys, bodin2025lincomblatgens}.

A natural strategy for avoiding these drawbacks is to decouple the two roles: use a strong pretrained generative model to capture the sparse structure of valid designs, while keeping the surrogate and acquisition function in data space, \rev{so that} domain knowledge and inductive biases \rev{can be included}.
To the best of our knowledge, COWBOYS \citep{moss2025cowboys} is the only prior BO method that follows this strategy, but it is tied to a VAE as the generative model and probability of improvement as the acquisition function.

In this paper, we propose \emph{acquisition function guided generation} ({AFGG}), which replaces acquisition function maximization with sampling from a distribution that tilts a frozen pretrained generative model toward designs favored by the acquisition function.
The construction is generic in both components: it requires only a generative model that supports guided sampling and any non-negative acquisition function that can be evaluated pointwise; the latter need neither be differentiable nor define a probability.
Since the discrete, combinatorial design space makes direct sampling from the tilted distribution difficult, we embed it into a diffusion process over the generative model's continuous space.
At each noise level, this yields a smooth posterior-averaged acquisition function that provides the gradients needed for guided sampling.
The resulting gradients can then guide sampling in the same spirit as guidance in diffusion models \citep{dhariwal2021diffusion, ho2022classifier}, which gives the method its name.

Our contributions are threefold:
\begin{itemize}[leftmargin=1.4em,itemsep=2pt,topsep=3pt]

\item \textbf{Acquisition function guided generation.}
We formulate the BO inner loop as sampling from an acquisition-tilted distribution that combines a frozen pretrained generative model with a surrogate defined directly in data space (\secref{sec:method-goal}).
We show that this target can be recovered exactly in the generative model's continuous space and that the resulting guidance term reduces to replacing the model's posterior mean by its acquisition-weighted counterpart (\secref{sec:method-estimator}).

\item \textbf{A unified construction for different generative models.}
We derive AFGG from a common diffusion-bridge construction and show how it applies to both categorical diffusion models and autoregressive normalizing flows (\secref{sec:method-two}).
The generator remains frozen throughout BO, while the surrogate and acquisition function stay in data space.

\item \textbf{State-of-the-art empirical performance.}
Across a broad suite of molecular design benchmarks under tight evaluation budgets, AFGG with the categorical diffusion model performs on par with leading BO methods.
More strikingly, AFGG with the flow model achieves state-of-the-art performance on most benchmarks and remains highly competitive on the rest (\secref{sec:experiments}).
\end{itemize}

\section{Background}
\label{sec:background}

\subsection{Bayesian Optimization over Discrete Spaces}
\label{sec:bo}

We consider Bayesian optimization over discrete, structured design spaces. 
In particular, we focus on sequence-valued settings and denote the design space as $\gX := \vocab^{L}$, the set of length-$L$ sequences over a discrete vocabulary $\vocab$ of size $V$.
We assume that valid or useful designs occupy a constrained and sparsely populated subset of the combinatorial space.
This setting covers, for example, molecules represented as SELFIES strings \citep{krenn2020selfies} and proteins represented as amino-acid sequences.
We seek the maximizer $w^\ast$ of an expensive-to-evaluate and black-box objective $f \colon \gX \to \R$, \rev{which provides no gradients and returns} measurements $y$ \rev{that may be} contaminated by Gaussian noise:
\begin{equation}
  w^\ast \in \argmax_{w\in\gX} ~f(w),
  \qquad
  y(w) = f(w)+\varepsilon_{\mathrm{obs}},\quad \varepsilon_{\mathrm{obs}}\sim\mathcal{N}(0,\sigma_{\mathrm{obs}}^2).
  \label{eq:bbo}
\end{equation}

Bayesian optimization (BO) fits a probabilistic surrogate to the observations collected so far and uses its predictive uncertainty to choose the next design to evaluate \citep{shahriari2016taking, garnett2023bayesian}. 
Throughout this work, we use Gaussian processes (GPs) \citep{rasmussen2006gaussian}, whose posterior mean and variance are available in closed form. 
However, our method is not tied to GPs and only requires a probabilistic surrogate from which an acquisition function can be computed or approximated.

Defining the surrogate directly on the discrete space allows domain knowledge to be incorporated through suitable inductive biases, such as domain-specific kernels. For example, in our molecular experiments we use the Tanimoto kernel \citep{ralaivola2005graph,tripp2023tanimoto} on binary Morgan fingerprints $\psi(w)\in \{0,1\}^{d_{\text{fp}}}$, where usually $d_{\text{fp}}=2048$ \citep{rogers2010extended},
\begin{equation}
k(w,w')
= \frac{\langle \psi(w),\psi(w')\rangle}
{\rev{\lVert\psi(w)\rVert}^2+\rev{\lVert\psi(w')\rVert}^2-\langle \psi(w),\psi(w')\rangle}
\in[0,1].
\label{eq}
\end{equation}

An acquisition function $\acq \colon \gX\to\R$ ranks candidate designs by balancing predicted performance against uncertainty.
We consider two standard choices, probability of improvement (PI) and expected improvement (EI), relative to the incumbent $y^\star$:
\begin{equation}
\acq^{\PI}(w)=\Pr\big[f(w)\ge y^\star\mid\dataset\big]\in[0,1],
\qquad
\acq^{\EI}(w)=\E\big[\max(f(w)-y^\star,0)\mid\dataset\big]\in[0,\infty),
\label{eq:pi-ei}
\end{equation}
where $\dataset=\{(w_j,y_j)\}_{j=1}^n$ denotes the observations collected so far.
Both are available in closed form under a GP \citep{ament2023unexpected}.
PI is a probability, whereas EI is non-negative but potentially unbounded; this distinction will matter for our method.
Standard BO selects the next design by maximizing the acquisition function $w_{\mathrm{next}}\in\argmax_{w\in\gX} ~ \acq(w)$.
For discrete, combinatorial $\gX$, this inner optimization is generally hard, and gradient-based methods for continuous BO do not apply directly.

\subsection{Generative Models over Sequences}
\label{sec:genmodels}

Our method requires a pretrained generative model over $\gX$ that supports guided sampling.
We consider two well-suited model classes for discrete sequences: a categorical diffusion model and an autoregressive normalizing flow.
Both connect the discrete design space $\gX = \vocab^{L}$ to a continuous embedding space $\R^{L\times d}$ in the same way: an embedding table $E\in\R^{|\vocab|\times d}$ (learned by the model) with rows $\rev{e}_v$, $v\in\vocab$\rev{,} maps a sequence to its embeddings, and a deterministic decode map $\tau\colon\R^{L\times d}\to\gX$ clamps each position to the closest token in cosine similarity,
\begin{equation}
  x = E(w) \triangleq \big(\rev{e}_{w_1},\dots,\rev{e}_{w_L}\big)\in\R^{L\times d},
  \qquad
  \tau(x)_i=\argmax_{v\in\vocab}\cos\big(x_i, \rev{e}_v \big),\quad i=1,\dots,L .
  \label{eq:embed-decode}
\end{equation}

\textbf{Continuous diffusion for categorical data (CDCD)}
\citep{dieleman2022continuous} runs a variance-exploding diffusion \citep{karras2022elucidating} on the embeddings,
\begin{equation}
  z_t = x + t\,\epsilon,\qquad \epsilon\sim\mathcal{N}(0,\rev{I}),
  \label{eq:cdcd-forward}
\end{equation}
so that the forward kernel is $q_t(z_t\mid x)=\mathcal N(z_t;x,t^2\rev{I})$. Samples are generated by integrating the reverse-time SDE or probability-flow ODE from $t=T$ to $t=0$~\citep{song2021score}\rev{,} which requires the score $\nabla_{z_t}\log p_t(z_t)$ of the noised marginal. Since $\nabla_{z_t}\log q_t(z_t\mid x)=(x-z_t)/t^2$ is affine in $x$, the score is the conditional score averaged over the denoising posterior, and in the discrete case, where each position of $x$ is one of $\vocab$'s embeddings, it has the form
\begin{equation}
  \nabla_{z_t}\log p_t(z_t)
  = \frac{\zhat-z_t}{t^2},
  \qquad
  \zhat(z_t,t)_i=\sum_{v=1}^{V} p(w_i=v\mid z_t)\,\rev{e}_v .
  \label{eq:score-interp}
\end{equation}
The score therefore depends on the joint denoising posterior only through its per-position marginals \citep{dieleman2022continuous}, and these are what must be learned: a transformer, trained with categorical cross-entropy on noised embeddings, outputs $L$ mass functions over $\vocab$,
\begin{equation}
  p_\theta(w_i = v\mid z_t,t)=\softmax\big(\mathrm{logits}_\theta(z_t,t)_i\big)_v ,
  \qquad i=1,\dots,L.
  \label{eq:cdcd-posterior}
\end{equation}
Substituting $p_\theta$ for $p$ in \eqref{eq:score-interp} gives the learned score.

\textbf{Autoregressive normalizing flow (SeqFlow)} \citep{lee2025nfbo} draws a Gaussian latent variable $z\sim\mathcal N(0,\rev{I})$, $z\in\R^{L\times d}$, and maps it through a diffeomorphism $g_\theta:\R^{L\times d}\to\R^{L\times d}$ to a sequence of continuous token embeddings, $x=g_\theta(z)$, which $\tau$ decodes by \eqref{eq:embed-decode}.
Because $g_\theta$ is a bijection and $\tau$ is deterministic, the resulting distribution over designs is exactly the pushforward of the latent Gaussian,
\begin{equation}
P_\theta(w)=\int_{A_w}\mathcal N(z;0,\rev{I})\,dz,\qquad A_w:=(\tau\circ g_\theta)^{-1}(w)\subset\R^{L\times d},
\label{eq:seqflow-pushforward}
\end{equation}
where the cells $\{A_w\}_{w\in\gX}$ partition the latent space.
We refer to \citet{lee2025nfbo} for details of the architecture and training procedure.

\section{Acquisition Function-Guided Generation}
\label{sec:method}

We now present AFGG, which steers a pretrained generative model toward designs with high acquisition values.
We first define the acquisition-tilted target distribution and explain why it is hard to sample it directly in data space (\secref{sec:method-goal}).
We then introduce a diffusion bridge that smooths this target and recovers informative gradients (\secref{sec:method-bridge}).
The resulting guidance term reduces to a difference of posterior means, which we estimate by sampling (\secref{sec:method-estimator}).
Finally, we show how the same construction applies to our two generative backbones, which differ only in how they provide the base score and posterior samples (\secref{sec:method-two}).

\subsection{The Sampling Problem}
\label{sec:method-goal}

As described in \secref{sec:background}, the generative model defines a density $p_\theta(x)$ on the continuous embedding space $\R^{L\times d}$ together with a deterministic decoder $\tau\colon\R^{L\times d}\to\gX$.
Each discrete design sequence $w\in\gX$ corresponds to the region of embedding space that decodes to it,
\begin{equation}
\tau^{-1}(w) \triangleq \{x\in\R^{L\times d} \colon \tau(x)=w\}.
\end{equation}
These regions partition the embedding space, and the probability of a sequence is the mass that $p_\theta$ assigns to its corresponding region.
Hence, the decoder pushes $p_\theta$ forward to a probability mass function $P_\theta$ over the discrete sequence space,
\begin{equation}
\label{eq:tau-pushforward}
P_\theta(w)
=
\int_{\tau^{-1}(w)}p_\theta(x)\,\mathrm{d}x .
\end{equation}

\begin{definition}[Tilted target]
\label{def:p_x0_tild}
Given the induced probability mass function $P_\theta$ over $\gX$ and a non-negative acquisition function $\ell\colon\gX\to\R_+$, we define the \emph{tilted target distribution}
\begin{equation}
\label{eq:tilde_prior}
\Pstar(w)
\triangleq
\frac{1}{Z}P_\theta(w)\ell(w),
\qquad
Z
\triangleq
\sum_{w\in\gX}P_\theta(w)\ell(w).
\end{equation}
\end{definition}

The tilted distribution combines the two components of our approach: $P_\theta$ places mass on sequences supported by the pretrained generative model, while $\ell$ shifts this mass toward sequences favored by the surrogate.
The requirement that $\ell$ be non-negative is mild, as a real-valued acquisition function $\tilde{\ell}$ can be turned into a positive tilt, for example through $\ell(w)=\exp( \tilde{\ell}(w))$.

This product form also has a natural interpretation through entropy-regularized control \citep{uehara2024finetuning}.
Indeed, $\Pstar$ is the solution of
\begin{equation}
\label{eq:variational}
\Pstar
=
\argmax_{P}
\left\{
\E_{P}\big[\log\ell(w)\big]
-
\KL\big(P\,\|\,P_{\theta}\big)
\right\}.
\end{equation}
The first term rewards designs with high acquisition values, while the KL term keeps the resulting distribution close to the pretrained generative model.
Equivalently, entropy-regularized control yields a target of the form $P_\theta(w)\exp(r(w))$, and our tilted distribution follows by taking $r(w)=\log\ell(w)$.

While sampling from the generative model $P_\theta$ is easy, the acquisition tilt changes this distribution such that directly sampling from $\Pstar$ would require handling the combinatorial space $\gX$, facing much the same difficulty as maximizing the acquisition function over $\gX$.
To sample from $\Pstar$, we instead return to the continuous space of the generative model.
Using \eqref{eq:tau-pushforward}, the corresponding tilted density is
\begin{equation}
\label{eq:tilted-embedding}
\pstar(x)
\triangleq
\frac{1}{Z}p_{\theta}(x)\ell(\tau(x)).
\end{equation}
Sampling $x\sim\pstar$ and decoding $w=\tau(x)$ yields exactly $w\sim\Pstar$; we prove this in \appref{proof:lifting}.

The main technical challenge is that $\ell(\tau(x))$ is constant within each cell $\tau^{-1}(w)$, so $\nabla_x\,\ell(\tau(x))=0$ almost everywhere.
LSBO avoids this difficulty by moving the surrogate to latent space, but in doing so inherits the limitations discussed above.
The rest of this section shows how to recover informative gradients while keeping the surrogate in data space.

\subsection{Lifting the Target to a Bridge}
\label{sec:method-bridge}

We recover informative gradients by introducing a family of smoothed versions of $\pstar$, indexed by $t\in[0,T]$, with $t=0$ corresponding to the tilted target and $t=T$ to the noise distribution.
Note that, although this smoothing construction is borrowed from diffusion models \citep{ho2020denoising, song2021score}, the generative model itself need not be a diffusion model; it only needs to support the guided sampling procedure developed below.

\begin{definition}[Bridge marginals]
\label{def:p_t}
Starting from the tilted density $\pstar$ and the original generative density $p_\theta$, we define their \emph{bridge marginals} at noise level $t$ as
\begin{equation}
    \pstar_t(z_t)
    =
    \int \pstar(x)\,q_t(z_t\mid x)\,\mathrm{d}x,
    \qquad
    p_t(z_t)
    =
    \int p_\theta(x)\,q_t(z_t\mid x)\,\mathrm{d}x ,
\end{equation}
\end{definition}
where
\begin{equation}
\label{eq:kernel-x}
q_t(z_t\mid x)
=
\mathcal{N}\big(z_t;\alpha_t x,\sigma_t^2 I\big)
\end{equation}
is a Gaussian noising kernel applied to the continuous variable $x$.

A key observation is that the acquisition tilt is preserved under this smoothing in a convenient form.
Let
\begin{equation}
p(x\mid z_t)
=
\frac{p_\theta(x)q_t(z_t\mid x)}{p_t(z_t)}
\end{equation}
denote the posterior over the clean variable $x$ under the original generative model.

\begin{proposition}[Tilted marginal]
\label{prop:tilted-marginal}
The tilted marginal can be written as
\begin{equation}
\label{eq:tilted-marginal}
\pstar_t(z_t)
\propto
p_t(z_t)\,\ell_t(z_t),
\qquad
\ell_t(z_t)
\triangleq
\E_{p(x\mid z_t)}\big[\ell(\tau(x))\big],
\end{equation}
where $\ell_t$ is the acquisition function averaged over the posterior uncertainty about the clean design sequence.
Moreover, the corresponding tilted posterior satisfies
\begin{equation}
\label{eq:tilted-posterior}
\pstar(x\mid z_t)
\propto
p(x\mid z_t)\,\ell(\tau(x)).
\end{equation}
\end{proposition}

\begin{proof}
The proof is in \appref{proof:tilted-marginal}.
\end{proof}

This factorization makes guidance convenient.
To sample from $\pstar$ by reversing the bridge, we need the score of the tilted marginal,
\begin{equation}
\label{eq:score-split}
\nabla_{z_t}\log\pstar_t(z_t)
=
\nabla_{z_t}\log p_t(z_t)
+
\nabla_{z_t}\log\ell_t(z_t).
\end{equation}
The first term is the score of the pretrained generative model at noise level $t$, while the second steers it toward regions favored by the acquisition function.
This has the same form as guidance in diffusion models \citep{dhariwal2021diffusion}.

Importantly, $\ell_t$ is no longer the piecewise-constant function $\ell\circ\tau$.
Instead, it averages $\ell(\tau(x))$ over the posterior $p(x\mid z_t)$, and can therefore vary smoothly with the continuous variable $z_t$.
Hence, $\nabla_{z_t}\log\ell_t(z_t)$ can provide an informative guidance signal even though $\nabla_x\ell(\tau(x))=0$ almost everywhere.

Evaluating \eqref{eq:score-split} therefore requires two ingredients: (i) the untilted score $\nabla_{z_t}\log p_t(z_t)$ and (ii) samples from the posterior $p(x\mid z_t)$.
The next section shows how these ingredients yield the guidance term, while \secref{sec:method-two} explains how they are obtained for two generative backbones.

\subsection{Guidance \rev{Is} a Tilted Mean Difference}
\label{sec:method-estimator}

The central object in \eqref{eq:score-split} is the guidance term $\nabla_{z_t}\log\ell_t(z_t)$.
For the Gaussian kernel \eqref{eq:kernel-x}, this term has a particularly simple form: it is proportional to the difference between the ordinary posterior mean and its acquisition-weighted counterpart.

\begin{definition}[Tilted estimate]
\label{def:tilted-estimate}
With $\pstar(x\mid z_t)$ as in \eqref{eq:tilted-posterior}, define
\begin{equation}
\ztilt
\triangleq
\E_{\pstar(x\mid z_t)}\big[x\big]
=
\frac{\E_{p(x\mid z_t)}\big[\ell(\tau(x))\,x\big]}
     {\E_{p(x\mid z_t)}\big[\ell(\tau(x))\big]},
\qquad
\zhat
\triangleq
\E_{p(x\mid z_t)}\big[x\big].
\end{equation}
\end{definition}

Thus, $\zhat$ is the posterior mean under the pretrained generative model, whereas $\ztilt$ gives more weight to posterior samples whose decoded designs have high acquisition value.

\begin{proposition}[Guidance score as tilted mean difference]
\label{prop:guidance-mean-swap}
Let the kernel of \defref{def:p_t} be the Gaussian kernel~(\plaineqref{eq:kernel-x}), with $\alpha_t,\sigma_t>0$.
Then
\begin{equation}
\label{eq:mean-swap}
\nabla_{z_t}\log\ell_t(z_t)
=
\frac{\alpha_t\big(\ztilt-\zhat\big)}{\sigma_t^2},
\qquad
\nabla_{z_t}\log p_t(z_t)
=
\frac{\alpha_t\,\zhat-z_t}{\sigma_t^2}.
\end{equation}
Substituting both expressions into \eqref{eq:score-split} gives
\begin{equation}
\label{eq:guided-score}
\nabla_{z_t}\log\pstar_t(z_t)
=
\frac{1}{\sigma_t^2}
\Big[
\big(\alpha_t\zhat-z_t\big)
+
\alpha_t\big(\ztilt-\zhat\big)
\Big]
=
\frac{\alpha_t\ztilt-z_t}{\sigma_t^2}.
\end{equation}
\end{proposition}

\begin{proof}
The proof is in \appref{proof:guidance-mean-swap}.
\end{proof}

The guided score therefore has exactly the same form as the original score, with the posterior mean $\zhat$ being replaced by its acquisition-weighted counterpart $\ztilt$.
This \emph{tilted mean difference} is a key simplification, as guidance does not require differentiating the acquisition function, but only estimating $\ztilt$.

In general, $\ztilt$ is not available in closed form because $\ell(\tau(x))$ need not \rev{factorize} across sequence positions, so we estimate it from samples $x^{(n)}\sim p(\cdot\mid z_t)$ by self-\rev{normalized} importance sampling.
For $N$ samples, let
\begin{equation}
\ell_n=\ell\big(\tau(x^{(n)})\big),
\qquad
h_n=\frac{\ell_n}{\sum_{m=1}^{N}\ell_m},
\qquad
\ztilt\approx\sum_{n=1}^{N}h_nx^{(n)}.
\end{equation}
This estimator is consistent (\appref{proof:snis}); guidance strength can further be controlled by a scale $\lambda$ \rev{(see} \eqref{eq:linear_tilt}\rev{)} or a temperature $\gamma$ on the importance weights \rev{(see} \eqref{eq:power_tilt}\rev{)}.

\subsection{Two \rev{Backbones}, \rev{One Bridge}}
\label{sec:method-two}
\label{sec:method-kernel}

\paragraph{CDCD.}
CDCD's forward process in \eqref{eq:cdcd-forward} is the kernel in \eqref{eq:kernel-x} with the variance-exploding schedule $\alpha_t\equiv1$, $\sigma_t=t$ \citep{karras2022elucidating}, so $\pstar_t=\pstar*\mathcal{N}(0,t^2I)$ and $z_t$ lives in the embedding space of \eqref{eq:embed-decode}.
The network provides both required quantities: the learned score \eqref{eq:score-interp} and the per-position posterior marginals \eqref{eq:cdcd-posterior}, from which posterior draws are taken under a mean-field approximation (\appref{app:sec-cdcd}).
At sufficiently large $T$, the added noise dominates both $p_\theta$ and the tilt,
\begin{equation}
  \pstar_T \approx p_T \approx \mathcal{N}(0,T^2 I),
  \label{eq:terminal-ve}
\end{equation}
so reversing the bridge with the PF-ODE or SDE driven by \eqref{eq:score-split} yields $x\sim\pstar$, which decodes to $w=\tau(x)\sim\Pstar$.

\paragraph{SeqFlow.}
SeqFlow has no built-in noising process.
Although $p_\theta(x)$ is tractable, applying \eqref{eq:kernel-x} in embedding space makes both the noised marginal $p_t$ and the posterior $p(x\mid z_t)$ intractable.
We therefore move the bridge to the flow latent $z=g_\theta^{-1}(x)$, whose untilted law is exactly $\mathcal{N}(0,I)$ by \eqref{eq:seqflow-pushforward}.
Since $g_\theta$ is a diffeomorphism, transporting the tilted density \eqref{eq:tilted-embedding} to the latent carries the same Jacobian factor $J$ as transporting $p_\theta$, so the factor cancels and only the tilt survives (\appref{proof:det-pushforward}),
\begin{equation}
  \pi^\star(z)
  \triangleq
  \pstar\big(g_\theta(z)\big)\,\big\lvert\det J_{g_\theta}(z)\big\rvert
  =
  \frac{1}{Z}\,\mathcal{N}(z;0,I)\,
  \ell\big(\tau(g_\theta(z))\big),
  \label{eq:push-forward}
\end{equation}
with $\tau(g_\theta(z))\sim\Pstar$ for $z\sim\pi^\star$.
As in \secref{sec:method-goal}, the tilt remains piecewise constant on the cells $A_w$ of \eqref{eq:seqflow-pushforward}, so we smooth it using
\begin{equation}
  q_t(z_t\mid z)
  =
  \mathcal{N}\big(z_t;\alpha_t z,\sigma_t^2I\big),
  \label{eq:kernel-z}
\end{equation}
with the variance-preserving schedule $\sigma_t=\sqrt{1-\alpha_t^2}$, $\alpha_t\in(0,1]$.
Gaussian conjugacy then gives both required quantities in closed form,
\begin{equation}
  \text{(i)}\ \nabla_{z_t}\log p_t(z_t)=-z_t,
  \qquad
  \text{(ii)}\ p(z\mid z_t)
  =
  \mathcal{N}\Big(z;\alpha_tz_t,(1-\alpha_t^2)I\Big).
  \label{eq:latent-posterior}
\end{equation}
The posterior is exact and network-free; the flow enters only through the tilt, evaluated once per draw and never differentiated.
Thus, \secref{sec:method-estimator} applies with $x$ replaced by $z$, while CDCD operates in embedding space.
\tabref{tab:easy-hard} summarizes the correspondence between the two backbones.

\subsection{Trust \rev{Region}}
\label{sec:trust-region}

Restricting search to a \rev{neighborhood} of the incumbent, the best design observed so far, is a common strategy for improving the sample efficiency of high-dimensional BO \citep{eriksson2019turbo}.
In AFGG, we implement this locality through the initialization of the guided sampling trajectories.
Both backbones start sampling from a Gaussian, $z_T\sim\mathcal{N}(0,\sigma^2 I)$, with $\sigma=T$ for CDCD \rev{(see} \eqref{eq:terminal-ve}\rev{)} and $\sigma=1$ for SeqFlow.
Instead, we wish to start closer to an anchor $z^\star$ that represents the incumbent in the same space.
However, sample quality deteriorates if the model must process an input that is unlikely under the $D$-dimensional Gaussian it expects, which concentrates its mass on a shell of radius $\sigma\sqrt{D}$ around the origin \citep[Thm.~2.9]{Blum_Hopcroft_Kannan_2020}.
We therefore sample
\begin{equation}
  z_T \sim \mathcal{N}\big(z_T;\, \rho z^\star,\, (1-\rho^2)\,\sigma^2 I\big),
  \label{eq:trust-region}
\end{equation}
where $\rho\in[0,1]$ controls how tightly the trust region concentrates around $z^\star$: $\rho=0$ recovers the unconditional initialization and $\rho=1$ starts every trajectory at $z^\star$.
Since $\E[\lVert z_T\rVert^2\mid z^\star]=\rho^2\lVert z^\star\rVert^2+(1-\rho^2)\,\sigma^2 D$, an anchor on the shell, $\lVert z^\star\rVert^2=\sigma^2 D$, gives $\E[\lVert z_T\rVert^2\mid z^\star]=\sigma^2 D$ for every $\rho$, so the initializations stay near the shell.

The complete AFGG optimization loop is described in \appref{sec:afgg_algo}.
\section{Experiments}
\label{sec:experiments}

We evaluate AFGG on the 25 molecular optimization objectives of \citet{gao2022sample} and \citet{gonzalezduque2024survey}, following the small-data protocol of \citet{moss2025cowboys}.
We consider both a \emph{sequential setting}, with 10 initial observations followed by 300 oracle evaluations, one per BO iteration, and a \emph{batch setting}, with 100 initial observations followed by 500 oracle evaluations in batches of five.
We instantiate AFGG with two different generative backbones, SeqFlow \citep{lee2025nfbo} and CDCD \citep{dieleman2022continuous}, as described in \secref{sec:method-two}, allowing us to test the same acquisition-guided construction with two substantially different generative models.
For SeqFlow, we use the released GuacaMol-trained checkpoint; since no pretrained CDCD checkpoint was available, we train one on the same dataset (\appref{app:sec-cdcd_training}).

We compare against five strong recent methods that are particularly relevant to our setting: CoBO \citep{lee2023cobo}, LOL-BO \citep{maus2022local}, COWBOYS \citep{moss2025cowboys}, InvBO \citep{chu2024inversion}, and NFBO \citep{lee2025nfbo}.
These include both the closest prior approach, COWBOYS, and NFBO, which uses the same SeqFlow generative backbone as one of our AFGG variants.
For the sequential setting, \appref{sec:additional_results} (\tabref{tab:afgg-all25-vs-benchmarks_full}) broadens the comparison to 11 baselines, additionally including CMA-ES, differential evolution (DE), genetic algorithms (GA), hill climbing (HC), RLBO\rev{~\citep{kirschner2019linebo}}, and TuRBO\rev{~\citep{eriksson2019turbo}}.
Unless stated otherwise, AFGG uses EI with the temperature tilt of \eqref{eq:power_tilt}, which performs best in our ablations (\rev{\appref{sec:hyperparameter-tuning}}).

\begin{figure}[t]
    \centering
    \includegraphics[width=\textwidth]{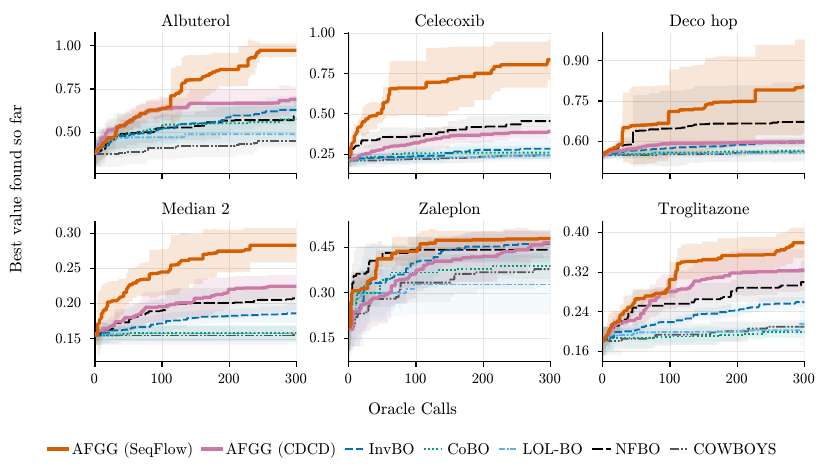}
    \caption{\textbf{Sequential setting.}
Optimization trajectories on six representative objectives, with 10 initial observations followed by 300 oracle evaluations, one per BO iteration.
Curves show the median over 5 seeds.}
    \label{fig:6_objfun_seq}
\end{figure}

\begin{table}[t]
\centering
\small
\resizebox{\textwidth}{!}{%
\begin{tabular}{lrrrrrrr}
\toprule
 & \multicolumn{5}{c}{baselines} & \multicolumn{2}{c}{AFGG} \\
\cmidrule(lr){2-6}\cmidrule(lr){7-8}
Objective & CoBO & InvBO & LOL-BO & NFBO & COWBOYS & CDCD & SeqFlow \\
\midrule
\texttt{albu\_sim} & 0.593 $\pm$ 0.111 & 0.629 $\pm$ 0.079 & 0.509 $\pm$ 0.065 & 0.593 $\pm$ 0.119 & 0.448 $\pm$ 0.037 & \textbf{0.691 $\pm$ 0.073} & \cellcolor[HTML]{D6F0D6}\textbf{0.975 $\pm$ 0.038} \\
\texttt{amlo\_mpo} & 0.523 $\pm$ 0.012 & 0.538 $\pm$ 0.050 & 0.516 $\pm$ 0.022 & 0.563 $\pm$ 0.095 & 0.472 $\pm$ 0.038 & 0.546 $\pm$ 0.039 & \textbf{0.565 $\pm$ 0.045} \\
\texttt{cele\_red} & 0.262 $\pm$ 0.041 & 0.286 $\pm$ 0.068 & 0.243 $\pm$ 0.035 & 0.455 $\pm$ 0.035 & 0.243 $\pm$ 0.020 & 0.392 $\pm$ 0.024 & \cellcolor[HTML]{D6F0D6}\textbf{0.836 $\pm$ 0.125} \\
\texttt{deco\_hop} & 0.565 $\pm$ 0.008 & 0.597 $\pm$ 0.019 & 0.560 $\pm$ 0.009 & 0.673 $\pm$ 0.147 & 0.559 $\pm$ 0.008 & 0.599 $\pm$ 0.010 & \textbf{0.803 $\pm$ 0.175} \\
\texttt{drd2} & 0.510 $\pm$ 0.370 & 0.963 $\pm$ 0.037 & 0.626 $\pm$ 0.184 & 0.971 $\pm$ 0.066 & 0.796 $\pm$ 0.243 & 0.730 $\pm$ 0.226 & \cellcolor[HTML]{D6F0D6}\textbf{1.000 $\pm$ 0.000} \\
\texttt{fexo\_mpo} & 0.710 $\pm$ 0.034 & 0.734 $\pm$ 0.027 & 0.716 $\pm$ 0.021 & 0.723 $\pm$ 0.032 & 0.683 $\pm$ 0.031 & \cellcolor[HTML]{D6F0D6}\textbf{0.763 $\pm$ 0.017} & \textbf{0.741 $\pm$ 0.016} \\
\texttt{gsk3b} & 0.362 $\pm$ 0.069 & 0.500 $\pm$ 0.042 & 0.336 $\pm$ 0.103 & 0.656 $\pm$ 0.174 & 0.418 $\pm$ 0.093 & 0.518 $\pm$ 0.088 & \cellcolor[HTML]{D6F0D6}\textbf{0.950 $\pm$ 0.037} \\
\texttt{iso\_c7} & 0.875 $\pm$ 0.145 & \cellcolor[HTML]{D6F0D6}\textbf{1.000 $\pm$ 0.000} & 0.760 $\pm$ 0.036 & 0.860 $\pm$ 0.092 & 0.919 $\pm$ 0.046 & 0.957 $\pm$ 0.049 & 0.981 $\pm$ 0.043 \\
\texttt{iso\_c9} & 0.577 $\pm$ 0.050 & 0.699 $\pm$ 0.147 & 0.630 $\pm$ 0.082 & 0.758 $\pm$ 0.024 & \cellcolor[HTML]{D6F0D6}\textbf{0.834 $\pm$ 0.020} & 0.722 $\pm$ 0.079 & 0.809 $\pm$ 0.081 \\
\texttt{jnk3} & 0.148 $\pm$ 0.031 & 0.206 $\pm$ 0.063 & 0.144 $\pm$ 0.035 & 0.412 $\pm$ 0.095 & 0.176 $\pm$ 0.032 & 0.216 $\pm$ 0.033 & \cellcolor[HTML]{D6F0D6}\textbf{0.778 $\pm$ 0.038} \\
\texttt{med\_1} & 0.178 $\pm$ 0.025 & 0.253 $\pm$ 0.037 & 0.175 $\pm$ 0.023 & 0.203 $\pm$ 0.024 & 0.192 $\pm$ 0.024 & \cellcolor[HTML]{D6F0D6}\textbf{0.346 $\pm$ 0.031} & \cellcolor[HTML]{D6F0D6}\textbf{0.352 $\pm$ 0.045} \\
\texttt{med\_2} & 0.158 $\pm$ 0.011 & 0.186 $\pm$ 0.014 & 0.158 $\pm$ 0.011 & 0.207 $\pm$ 0.018 & 0.156 $\pm$ 0.012 & \cellcolor[HTML]{D6F0D6}\textbf{0.225 $\pm$ 0.016} & \cellcolor[HTML]{D6F0D6}\textbf{0.283 $\pm$ 0.024} \\
\texttt{mest\_sim} & 0.421 $\pm$ 0.064 & 0.531 $\pm$ 0.054 & 0.384 $\pm$ 0.059 & 0.555 $\pm$ 0.125 & 0.404 $\pm$ 0.023 & \textbf{0.593 $\pm$ 0.051} & \textbf{0.658 $\pm$ 0.114} \\
\texttt{osim\_mpo} & 0.761 $\pm$ 0.013 & 0.782 $\pm$ 0.024 & 0.767 $\pm$ 0.015 & \textbf{0.803 $\pm$ 0.019} & 0.749 $\pm$ 0.023 & 0.778 $\pm$ 0.023 & 0.789 $\pm$ 0.034 \\
\texttt{peri\_mpo} & 0.487 $\pm$ 0.049 & 0.447 $\pm$ 0.054 & 0.416 $\pm$ 0.079 & 0.434 $\pm$ 0.015 & 0.344 $\pm$ 0.034 & 0.432 $\pm$ 0.031 & \textbf{0.502 $\pm$ 0.017} \\
\texttt{rano\_mpo} & 0.688 $\pm$ 0.027 & 0.742 $\pm$ 0.033 & 0.684 $\pm$ 0.042 & 0.696 $\pm$ 0.053 & 0.708 $\pm$ 0.040 & \textbf{0.757 $\pm$ 0.034} & \cellcolor[HTML]{D6F0D6}\textbf{0.785 $\pm$ 0.007} \\
\texttt{logp} & \cellcolor[HTML]{D6F0D6}\textbf{65.866 $\pm$ 10.720} & \textbf{49.799 $\pm$ 19.126} & 34.741 $\pm$ 5.318 & 25.055 $\pm$ 8.705 & 15.443 $\pm$ 3.623 & 22.684 $\pm$ 2.667 & 35.734 $\pm$ 10.360 \\
\texttt{qed} & 0.928 $\pm$ 0.004 & 0.942 $\pm$ 0.006 & 0.934 $\pm$ 0.003 & 0.943 $\pm$ 0.008 & 0.898 $\pm$ 0.034 & 0.940 $\pm$ 0.006 & \cellcolor[HTML]{D6F0D6}\textbf{0.945 $\pm$ 0.002} \\
\texttt{sa} & \cellcolor[HTML]{D6F0D6}\textbf{9.001 $\pm$ 0.347} & 8.393 $\pm$ 0.543 & 8.214 $\pm$ 0.711 & 7.538 $\pm$ 0.452 & 8.020 $\pm$ 0.363 & 8.349 $\pm$ 0.606 & 8.560 $\pm$ 0.220 \\
\texttt{scaf\_hop} & 0.441 $\pm$ 0.029 & 0.501 $\pm$ 0.030 & 0.430 $\pm$ 0.020 & 0.510 $\pm$ 0.062 & 0.442 $\pm$ 0.027 & 0.500 $\pm$ 0.028 & \cellcolor[HTML]{D6F0D6}\textbf{0.583 $\pm$ 0.070} \\
\texttt{sita\_mpo} & 0.255 $\pm$ 0.096 & \cellcolor[HTML]{D6F0D6}\textbf{0.330 $\pm$ 0.032} & 0.228 $\pm$ 0.087 & 0.290 $\pm$ 0.030 & \textbf{0.352 $\pm$ 0.090} & 0.108 $\pm$ 0.092 & 0.295 $\pm$ 0.056 \\
\texttt{thio\_red} & 0.270 $\pm$ 0.025 & 0.308 $\pm$ 0.017 & 0.255 $\pm$ 0.011 & 0.376 $\pm$ 0.053 & 0.261 $\pm$ 0.010 & \textbf{0.395 $\pm$ 0.051} & \cellcolor[HTML]{D6F0D6}\textbf{0.475 $\pm$ 0.060} \\
\texttt{trog\_red} & 0.205 $\pm$ 0.004 & 0.260 $\pm$ 0.010 & 0.218 $\pm$ 0.044 & 0.300 $\pm$ 0.044 & 0.210 $\pm$ 0.014 & \textbf{0.324 $\pm$ 0.043} & \cellcolor[HTML]{D6F0D6}\textbf{0.380 $\pm$ 0.031} \\
\texttt{vals\_sma} & 0.000 $\pm$ 0.000 & 0.000 $\pm$ 0.000 & 0.000 $\pm$ 0.000 & \textbf{0.015 $\pm$ 0.034} & 0.000 $\pm$ 0.000 & 0.000 $\pm$ 0.000 & 0.000 $\pm$ 0.000 \\
\texttt{zale\_mpo} & 0.404 $\pm$ 0.055 & 0.461 $\pm$ 0.047 & 0.348 $\pm$ 0.065 & 0.442 $\pm$ 0.042 & 0.378 $\pm$ 0.044 & \textbf{0.466 $\pm$ 0.037} & \textbf{0.480 $\pm$ 0.026} \\
\bottomrule
\end{tabular}}
\caption{\textbf{Sequential benchmark (partial results).} Final best score on all 25 PMO objectives after 300 oracle calls (10 initial points, 1 point per BO iteration), mean $\pm$ sample standard deviation over 5 seeds. For readability the table shows five of the eleven baselines; \tabref{tab:afgg-all25-vs-benchmarks_full} gives the full comparison. AFGG is shown in two instantiations: \textbf{CDCD} and \textbf{SeqFlow}. Every cell is judged against the opposing group. \textbf{Bold}: the cell's mean exceeds every mean in the opposing group. \colorbox[HTML]{D6F0D6}{Green, shaded}: lead exceeds the cell's own sample standard deviation; bold without shading denotes a smaller lead.
Objective abbreviations are defined in \tabref{tab:obj-abbrev}.}
\label{tab:afgg-all25-vs-benchmarks_partial}
\end{table}

\begin{figure}[t]
    \centering
    \includegraphics[width=\textwidth]{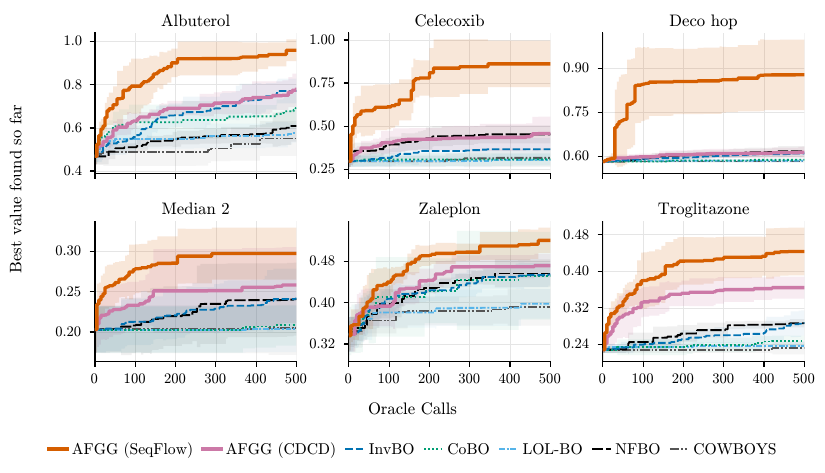}
    \caption{\textbf{Batch setting, uncoupled generation.}
Optimization trajectories on the same six representative objectives, with 100 initial observations followed by 500 oracle evaluations in batches of five.
Curves show the median over 5 seeds.}
    \label{fig:6_objfun_batch}
\end{figure}

\rev{\figref{fig:6_objfun_seq}} shows the optimization trajectories on six representative objectives in the sequential setting.
AFGG with SeqFlow makes particularly strong gains on several of these tasks, while the CDCD variant remains broadly competitive with the leading BO baselines.
To verify that this pattern is not specific to the selected objectives, \tabref{tab:afgg-all25-vs-benchmarks_partial} reports the final performance on all 25 objectives against the five closest BO baselines.

\rev{\figref{fig:6_objfun_batch}} repeats the comparison in the batch setting on the same six representative objectives.
The advantage of the SeqFlow variant persists when $q=5$ candidates are selected per BO round.
\rev{\appref{sec:additional_results}} reports results on all 25 objectives for both uncoupled and coupled batch constructions.

Across both settings, the two AFGG backbones show a clear and consistent pattern.
AFGG with CDCD performs on par with the strongest BO baselines, whereas AFGG with SeqFlow achieves the best result on most objectives and remains competitive on the others.
The same ordering across the sequential and batch settings suggests that the stronger generative backbone translates into substantially stronger optimization performance.
\textbf{Taken together, these results demonstrate state-of-the-art performance for AFGG with SeqFlow across this benchmark suite: it achieves the best result on most objectives and remains close to the strongest baseline on the others.}

Additional analyses in \appref{sec:generative-prior} isolate the contribution of the generative backbone and relate optimization performance to the quality of its unconditional samples.
\section{Related \rev{Work}}
\label{sec:related}

AFGG is related to \emph{inference-time alignment} of pretrained generative models, which likewise samples from a distribution obtained by tilting a generative model toward a reward \citep{krishnamoorthyddom, li2025svdd, uehara2024finetuning, uehara2025tutorial}.
The key difference is that inference-time alignment typically uses a fixed reward, whereas in Bayesian optimization the acquisition function changes after every new observation.
AFGG therefore applies the same general idea of guided generation in a setting where the tilt is continually updated by the surrogate.

Another line of related work is \emph{active generation}\rev{,} where the generative model is adapted repeatedly as new data arrive, for example by conditioning or variationally fitting it to an acquisition-weighted target \citep{brookes2019conditioning, steinberg2025variational, yun2025dibo, oliveira2026genbo}; the same approach can also be applied to infilling problems, where the generated sequence has fixed length \citep{stanton2022lambo, gruver2023protein}.
AFGG takes a different route: the pretrained generator remains frozen, while only the surrogate and acquisition function are updated during BO.
This allows strong pretrained generative models to be reused throughout optimization without repeated retraining or fine-tuning.

Most generative BO methods place the surrogate in the learned latent space \citep{tripp2020sample, maus2022local, lee2023cobo, chu2024inversion, lee2025nfbo}, which makes inductive biases defined naturally on the original designs harder to use directly.
To the best of our knowledge, COWBOYS \citep{moss2025cowboys} is the only prior BO method that similarly combines a frozen generative model with a surrogate defined in data space.
It is therefore the closest precursor to AFGG, but its construction is tied to a VAE and to probability of improvement.
In contrast, AFGG only requires a generative model that supports guided sampling and a non-negative acquisition function that can be evaluated pointwise.
The acquisition function need neither be differentiable nor define a probability, allowing, for example, EI to be used directly.

\section{Conclusion}
\label{sec:conclusion}

We introduced AFGG, which replaces acquisition-function maximization with sampling from an acquisition-tilted generative distribution.
A diffusion bridge makes this target amenable to guided sampling while leaving the surrogate and acquisition function in data space.
The generator remains frozen throughout BO, and the acquisition function only needs to be non-negative and evaluable pointwise.
On the molecular benchmarks studied here, AFGG with CDCD is competitive with strong recent BO methods, while AFGG with SeqFlow achieves the best result on most objectives and remains close to the strongest baselines on the rest.
The gap between the two variants further shows that the quality of the generative backbone matters directly for optimization performance.
This makes stronger pretrained generators a natural next target for AFGG.
In particular, protein foundation models \citep{af3, boltz2, rfdiffusion} could potentially serve as frozen backbones.
Applying AFGG to such models will mainly require efficient approximations to the posterior $p(x \mid z_t)$ used for guidance.

\section{AI \rev{Use Statement}}
Large language models were used as an assistive tool during the preparation of this work, under author supervision, in three roles:
\begin{itemize}
    \item \textbf{Writing.} Copy-editing of author-written text for clarity and concision, and suggestions for \rev{reorganizing} existing material. No section was drafted from scratch by a model; all technical claims, derivations and experimental conclusions are the authors'.
    \item \textbf{Related work.} Surfacing candidate references and drafting initial summaries of them. Every cited work was subsequently read by the authors, and all attributions and comparisons were verified against the primary sources.
    \item \textbf{Code.} Inspecting the released implementations of the baselines. \rev{Standardizing} the implementation and evaluation protocol across methods required modifications to the original code, and an AI assistant was used to help locate and implement them. Every deviation from an original implementation is documented, with the corresponding line numbers.
\end{itemize}
The authors take full responsibility for the entire content of this paper, including any text, code or references produced with AI assistance.

\section{Reproducibility \rev{Statement}}
An \rev{anonymized} implementation is available at \url{https://anonymous.4open.science/r/afgg}. It covers the guided samplers of \algref{alg:cdcd-guided} and \algref{alg:seqflow-guided}, the BO loop, and our modified baseline code, together with the configuration files and random seeds needed to reproduce every table and figure reported here. The guidance construction is stated in full in \secref{sec:method} and proved in \appref{app:proofs}. Hyperparameters and training details for the CDCD generator are given in \tabref{tab:hyp-config-cdcd} and \appref{app:sec-cdcd_training}; the SeqFlow backbone is the pretrained model of \citet{lee2025nfbo}, used without modification. 




\section{Acknowledgments}
 
This work was supported by the Austrian Science Fund (FWF) through the doc.funds ``BiotechPredict'' (10.55776/DOC2686925) and the Cluster of Excellence ``Bilateral AI'' (10.55776/COE12).
Additional support was provided by the TU Graz Leadprojekt ``DigiBioTech'' and Lancaster University’s Mathematics for AI in
Real-world Systems E3 grant.


\bibliography{iclr2027_conference}
\bibliographystyle{iclr2027_conference}


\appendix
\section{Proofs}
\label{app:proofs}

\subsection{Tilting \rev{Commutes} with \rev{Decoding}}
\label{proof:lifting}

\begin{proposition}[Lifting]
\label{prop:lifting}
Let $\pstar$ be the tilted density \eqref{eq:tilted-embedding} and $\Pstar$ the tilted target of
\defref{def:p_x0_tild}. Then a draw $x\sim\pstar$ returns $\tau(x)\sim\Pstar$.
\end{proposition}

\begin{proof}
Write $Z=\int_{\R^{L\times d}} p_\theta(x)\,\ell(\tau(x))\,\mathrm{d}x$ for the \rev{normalizing} constant of
$\pstar$. On the cell $\tau^{-1}(w)$ the decode map is constant by definition, $\tau(x)=w$, so
$\ell(\tau(x))=\ell(w)$ factors out of the integral:
\begin{equation}
\label{eq:lifting-cell}
\begin{aligned}
&\Pr_{x\sim\pstar}\big[\tau(x)=w\big]
=\int_{\tau^{-1}(w)}\pstar(x)\,\mathrm{d}x
=\frac{1}{Z}\int_{\tau^{-1}(w)}p_\theta(x)\,\ell(\tau(x))\,\mathrm{d}x
\\[2pt]
&=\frac{\ell(w)}{Z}\int_{\tau^{-1}(w)}p_\theta(x)\,\mathrm{d}x
=\frac{P_\theta(w)\,\ell(w)}{Z} .
\end{aligned}
\end{equation}
The same decomposition identifies the constant. Since the cells $\{\tau^{-1}(w)\}_{w\in\gX}$ partition
$\R^{L\times d}$,
\begin{equation}
\label{eq:lifting-constant}
Z=\sum_{w\in\gX}\int_{\tau^{-1}(w)}p_\theta(x)\,\ell(\tau(x))\,\mathrm{d}x
 =\sum_{w\in\gX}P_\theta(w)\,\ell(w),
\end{equation}
which is exactly the \rev{normalizing} constant of $\Pstar$ in \eqref{eq:tilde_prior}. Hence
$\tau(x)\sim\Pstar$ for $x\sim\pstar$.
\end{proof}

Note that only the constancy of $\ell\circ\tau$ on cells was used; no property of $p_\theta$ and no
regularity of $\tau$ beyond measurability is required. This is what licenses tilting in the continuous
space and decoding afterwards, and for SeqFlow it is applied a second time with
$\tau$ replaced by $\tau\circ g_\theta$ (\appref{proof:det-pushforward}).

\subsection{Proof of \propref{prop:tilted-marginal}}
\label{proof:tilted-marginal}

\paragraph{Statement.} With $\ell_t(z_t)\triangleq\E_{p(x\mid z_t)}\big[\ell(\tau(x))\big]$,
\propref{prop:tilted-marginal} asserts that the tilted marginal and the tilted posterior factor as
\begin{equation*}
\pstar_t(z_t)\;\propto\;p_t(z_t)\,\ell_t(z_t),
\qquad
\pstar(x\mid z_t)\;\propto\;p(x\mid z_t)\,\ell(\tau(x)),
\end{equation*}
the first being \eqref{eq:tilted-marginal}.

\begin{proof}
Expand $\pstar_t$ using \defref{def:p_x0_tild} and \defref{def:p_t}:
\begin{equation}
\pstar_t(z_t)=\int q_t(z_t\mid x)\, \pstar(x)\, \mathrm{d}x
\propto \int q_t(z_t\mid x)\, p_\theta(x)\, \ell(\tau(x))\, \mathrm{d}x .
\end{equation}
By Bayes' rule, $q_t(z_t\mid x)\,p_\theta(x)=p(x\mid z_t)\,p_t(z_t)$, so
\begin{equation}
\pstar_t(z_t) \propto p_t(z_t)\int p(x\mid z_t)\,\ell(\tau(x))\,\mathrm{d}x
 = p_t(z_t)\,\ell_t(z_t),
\end{equation}
with $\ell_t(z_t)=\E_{p(x\mid z_t)}[\ell(\tau(x))]$, which is \eqref{eq:tilted-marginal}.

For the second claim, apply Bayes' rule to the tilted pair $\big(\pstar,q_t\big)$ and substitute both
expansions above:
\begin{equation}
\pstar(x\mid z_t)=\frac{\pstar(x)\,q_t(z_t\mid x)}{\pstar_t(z_t)}
 \propto \frac{p_\theta(x)\,\ell(\tau(x))\,q_t(z_t\mid x)}{p_t(z_t)\,\ell_t(z_t)}
=\underbrace{\frac{p_\theta(x)\,q_t(z_t\mid x)}{p_t(z_t)}}_{= \;p(x\mid z_t)}\cdot
\frac{\ell(\tau(x))}{\ell_t(z_t)} .
\end{equation}
Since $\ell_t(z_t)$ does not depend on $x$, it is absorbed into the \rev{normalizing} constant, giving
$\pstar(x\mid z_t)\propto p(x\mid z_t)\,\ell(\tau(x))$, which is \eqref{eq:tilted-posterior}.
\end{proof}

\subsection{Change of \rev{Variables} to the \rev{Flow Latent}}
\label{proof:det-pushforward}
\secref{sec:method-two} runs SeqFlow in its latent $z=g_\theta^{-1}(x)$ rather than in $x$. This requires two facts: the tilted target \eqref{eq:tilted-embedding}, pushed through $g_\theta^{-1}$, is again a Gaussian tilted by the acquisition, now composed with the decoder, \eqref{eq:push-forward}; and decoding a draw of it recovers $\Pstar$. Both follow from the change-of-variables formula, because the Jacobian of the flow enters the tilted and the untilted density in the same way.

\begin{proposition}[Flow-latent tilt]
\label{prop:det-pushforward}
Let $g_\theta$ be a diffeomorphism of $\R^{L\times d}$ such that $z\sim\mathcal{N}(0,\rev{I})$ implies
$g_\theta(z)\sim p_\theta$, as in \eqref{eq:seqflow-pushforward}, and let
$\pi^\star(z)\triangleq\pstar\big(g_\theta(z)\big)\,\lvert J(z)\rvert$, with $J(z)$ the Jacobian
determinant of $g_\theta$ at $z$. Then
\begin{equation*}
\pi^\star(z)=\frac{1}{Z}\,\mathcal{N}(z;0,\rev{I})\;\ell\big(\tau(g_\theta(z))\big),
\end{equation*}
which is \eqref{eq:push-forward}, and $z\sim\pi^\star$ implies $\tau(g_\theta(z))\sim\Pstar$.
\end{proposition}

\begin{proof}
Since $g_\theta$ is a diffeomorphism, the change-of-variables formula applies to $\pstar$ and to $p_\theta$ with the \emph{same} Jacobian factor. Writing $J(z)$ for the Jacobian determinant of $g_\theta$ at $z$,
\begin{equation}
\label{eq:cov-pair}
\pi^\star(z)=\pstar\big(g_\theta(z)\big)\,\lvert J(z)\rvert,
\qquad
\mathcal{N}(z;0,\rev{I})=p_\theta\big(g_\theta(z)\big)\,\lvert J(z)\rvert ,
\end{equation}
the second because $g_\theta$ transports $\mathcal{N}(0,\rev{I})$ to $p_\theta$, as in \eqref{eq:seqflow-pushforward}. Substituting $\pstar=\tfrac1Z\,p_\theta\cdot(\ell\circ\tau)$ from \eqref{eq:tilted-embedding} into the first and \rev{recognizing} the second,
\begin{equation}
\label{eq:det-pushforward}
\pi^\star(z)
=\frac{1}{Z}\,
\underbrace{p_\theta\big(g_\theta(z)\big)\,\lvert J(z)\rvert}_{=\;\mathcal{N}(z;0,\rev{I})}\;
\ell\big(\tau(g_\theta(z))\big),
\end{equation}
which is \eqref{eq:push-forward}: the Jacobian is common to both densities and cancels in the ratio $\pi^\star/\mathcal{N}(0,\rev{I})=(\pstar/p_\theta)\circ g_\theta$. Finally, $z\sim\pi^\star$ gives $g_\theta(z)\sim\pstar$ by the first identity of \eqref{eq:cov-pair}, so $\tau(g_\theta(z))\sim\Pstar$ by \appref{proof:lifting}.
\end{proof}

\subsection{Proof of \propref{prop:guidance-mean-swap}}
\label{proof:guidance-mean-swap}

\paragraph{Statement.} For the Gaussian kernel \eqref{eq:kernel-x} with $\alpha_t,\sigma_t>0$, and with
$\zhat,\ztilt$ as in \defref{def:tilted-estimate}, \propref{prop:guidance-mean-swap} asserts
\begin{equation*}
\nabla_{z_t}\log\ell_t(z_t)=\frac{\alpha_t\big(\ztilt-\zhat\big)}{\sigma_t^2},
\qquad
\nabla_{z_t}\log p_t(z_t)=\frac{\alpha_t\,\zhat-z_t}{\sigma_t^2},
\end{equation*}
so that the guided score \eqref{eq:score-split} is $\big(\alpha_t\ztilt-z_t\big)/\sigma_t^2$, i.e.\
\eqref{eq:mean-swap} and \eqref{eq:guided-score}.

\begin{proof}
We differentiate $\log\ell_t$ directly and exhibit the two means whose $z_t$ cancels. Write
\begin{equation}
\label{eq:I-definition}
I(z_t) \triangleq \int \ell(\tau(x))\, q_t(z_t\mid x)\, p_\theta(x)\, \mathrm{d}x .
\end{equation}
By Bayes' rule, $p(x\mid z_t)=q_t(z_t\mid x)\,p_\theta(x)/p_t(z_t)$, so the definition of $\ell_t$ in
\propref{prop:tilted-marginal} reads
\begin{equation}
\label{eq:ell-q-ratio}
\ell_t(z_t)=\int \ell(\tau(x))\,p(x\mid z_t)\,\mathrm{d}x
=\frac{1}{p_t(z_t)}\int \ell(\tau(x))\,q_t(z_t\mid x)\,p_\theta(x)\,\mathrm{d}x
=\frac{I(z_t)}{p_t(z_t)},
\end{equation}
and therefore
\begin{equation}
\label{eq:log-ell-split}
\nabla_{z_t}\log\ell_t(z_t) = \nabla_{z_t}\log I(z_t) - \nabla_{z_t}\log p_t(z_t).
\end{equation}
Since $q_t$ is the
Gaussian kernel \eqref{eq:kernel-x}, $\nabla_{z_t}\log q_t(z_t\mid x)=(\alpha_t\,x-z_t)/\sigma_t^2$,
so by the log-derivative rule $\nabla_{z_t} q_t=q_t\,(\alpha_t\,x-z_t)/\sigma_t^2$. Consequently
\begin{equation}
\label{eq:grad-log-I}
\begin{aligned}
\nabla_{z_t}\log I(z_t)
&=\frac{\nabla_{z_t} I(z_t)}{I(z_t)}
=
\frac{\displaystyle\int \ell(\tau(x))\,\frac{\alpha_t\,x-z_t}{\sigma_t^2}\,q_t(z_t\mid x)\,p_\theta(x)\,\mathrm{d}x}
{\displaystyle\int \ell(\tau(x))\,q_t(z_t\mid x)\,p_\theta(x)\,\mathrm{d}x}\\[6pt]
&=
\frac{\displaystyle\int \ell(\tau(x))\,\frac{\alpha_t\,x-z_t}{\sigma_t^2}\,p(x\mid z_t)\,\mathrm{d}x}
{\displaystyle\int \ell(\tau(x))\,p(x\mid z_t)\,\mathrm{d}x}\\[6pt]
&=
\frac{1}{\sigma_t^2}\cdot
\frac{\alpha_t\displaystyle\int x\,\ell(\tau(x))\,p(x\mid z_t)\,\mathrm{d}x\;-\;z_t\,\ell_t(z_t)}{\ell_t(z_t)}\\[6pt]
&=
\frac{1}{\sigma_t^2}\left(\alpha_t\int x\,\frac{\ell(\tau(x))\,p(x\mid z_t)}{\ell_t(z_t)}\,\mathrm{d}x\;-\;z_t\right)\\[6pt]
&=
\frac{1}{\sigma_t^2}\left(\alpha_t\int x\,\pstar(x\mid z_t)\,\mathrm{d}x\;-\;z_t\right)
=
\frac{\alpha_t\,\ztilt - z_t}{\sigma_t^2}.
\end{aligned}
\end{equation}

The second term of \eqref{eq:log-ell-split} follows from the same steps applied to
$p_t(z_t)=\int q_t(z_t\mid x)\,p_\theta(x)\,\mathrm{d}x$. Now the denominator is $p_t(z_t)$ itself, so
dividing by it gives $p(x\mid z_t)$ at once, and the second member carries
$\int p(x\mid z_t)\,\mathrm{d}x=1$ in place of $\ell_t(z_t)$:
\begin{equation}
\label{eq:tweedie-general}
\begin{aligned}
\nabla_{z_t}\log p_t(z_t)
&=\frac{\nabla_{z_t}p_t(z_t)}{p_t(z_t)}
=\int \frac{\alpha_t\,x-z_t}{\sigma_t^2}\,\frac{q_t(z_t\mid x)\,p_\theta(x)}{p_t(z_t)}\,\mathrm{d}x\\[6pt]
&=\frac{1}{\sigma_t^2}\left(\alpha_t\int x\,p(x\mid z_t)\,\mathrm{d}x\;-\;z_t\int p(x\mid z_t)\,\mathrm{d}x\right)
=\frac{\alpha_t\,\zhat - z_t}{\sigma_t^2},
\end{aligned}
\end{equation}
where $\zhat=\E_{p(\cdot\mid z_t)}[x]$; this is Tweedie's formula for this kernel and is the
untilted score quoted in \propref{prop:guidance-mean-swap}.

Substituting \eqref{eq:grad-log-I} and \eqref{eq:tweedie-general} into \eqref{eq:log-ell-split}, the $z_t$
terms --- identical in both --- cancel:
\begin{equation}
\nabla_{z_t}\log\ell_t(z_t)=\frac{\alpha_t\,\ztilt - z_t}{\sigma_t^2}-\frac{\alpha_t\,\zhat - z_t}{\sigma_t^2}
=\frac{\alpha_t\big(\ztilt - \zhat\big)}{\sigma_t^2},
\end{equation}
which is \eqref{eq:mean-swap}. Adding \eqref{eq:tweedie-general} back recovers \eqref{eq:guided-score}.
\end{proof}

\subsection{Estimating the \rev{Tilted Mean}}
\label{proof:snis}
By \propref{prop:guidance-mean-swap}, the guidance term depends on the frozen model only through the tilted posterior mean $\ztilt$ of \defref{def:tilted-estimate}. Its untilted counterpart $\zhat$ comes with the score (i), but $\ztilt$ is a ratio of two expectations under $p(\cdot\mid z_t)$ with $\ell$ a black box, and has no closed form. This section gives the estimator used in \secref{sec:method-estimator}, how its draws are obtained for each backbone, and the two parameters that set the strength of the tilt.

\begin{proposition}[Self-\rev{normalized} importance-sampling (SNIS) estimator]
\label{prop:snis}
Draw $N$ i.i.d.\ samples $x^{(n)}\sim p(x\mid z_t)$ and write $\ell_n \triangleq \ell(\tau(x^{(n)}))$.
Then
\begin{equation}
\label{eq:snis}
\ztilt \approx \sum_{n=1}^{N} h_n\,x^{(n)},
\qquad
h_n = \frac{\ell_n}{\sum_{m}\ell_m},
\qquad
\sum_n h_n = 1 .
\end{equation}
\end{proposition}

\begin{proof}
We start by rewriting \defref{def:tilted-estimate}. By \eqref{eq:tilted-posterior},
$\pstar(x\mid z_t)=p(x\mid z_t)\,\ell(\tau(x))/\ell_t(z_t)$, hence
\begin{equation}
\label{eq:xtilt-ratio}
\ztilt
= \frac{1}{\ell_t(z_t)}\int x\,p(x\mid z_t)\,\ell(\tau(x))\,\mathrm{d}x
= \frac{\E_{p(\cdot\mid z_t)}\big[\ell(\tau(x))\,x\big]}
       {\E_{p(\cdot\mid z_t)}\big[\ell(\tau(x))\big]} ,
\end{equation}
the last step by the definition of $\ell_t$ in \propref{prop:tilted-marginal}. Neither expectation
admits a closed form: $\ell$ is a scalar function of the whole decoded object and does not \rev{factorize}
over the coordinates of $x$, so numerator and denominator both couple all $L$ positions.
Estimating the two expectations of \eqref{eq:xtilt-ratio} by Monte Carlo, using the same $N$ draws
$x^{(n)}\sim p(\cdot\mid z_t)$ in both, gives
\begin{equation}
\label{eq:mc-weights}
\ztilt
\approx
\frac{\tfrac1N\sum_{n=1}^{N}\ell_n\,x^{(n)}}{\tfrac1N\sum_{m=1}^{N}\ell_m}
= \sum_{n=1}^{N}\underbrace{\frac{\ell_n}{\sum_{m}\ell_m}}_{h_n}x^{(n)} ,
\qquad \sum_n h_n = 1 .
\end{equation}
The estimator is a ratio of two unbiased estimators sharing the same draws, hence consistent but
biased at finite $N$, with bias $O(1/N)$; the self-\rev{normalization} in $h_n$ is what removes the
unknown $\ell_t(z_t)$.
\end{proof}

\paragraph{Two knobs, and they are not the same operation.}\label{app:knobs} In \eqref{eq:tilde_prior} the acquisition enters with unit exponent; in practice the strength of the tilt is a tuning parameter. It can be adjusted at two points of the estimator: after the tilted mean has been formed, or inside the weights. The first, $\lambda$, scales the guidance
term of \eqref{eq:score-split} after the tilted mean has been formed,
$\lambda\nabla_{z_t}\log\ell_t=\nabla_{z_t}\log\ell_t^{\lambda}$:
\begin{equation}
\label{eq:linear_tilt}
\nabla_{z_t}\log p_t + \lambda\,\nabla_{z_t}\log\ell_t
= \frac{\alpha_t\big(\zhat+\lambda(\ztilt-\zhat)\big)-z_t}{\sigma_t^2}
= \frac{\alpha_t\,\ztiltg-z_t}{\sigma_t^2}.
\end{equation}
This leaves the draws and the weights of \eqref{eq:snis} untouched and moves only the
target, along the segment from $\zhat$ to $\ztilt$. Powering $\ell$ first and \rev{marginalizing} second changes only the
weights,
\begin{equation}
\label{eq:power_tilt}
h_n^{(\gamma)}=\softmax_n\Big(\tfrac1\gamma\log\ell_n\Big)
=\frac{\ell_n^{1/\gamma}}{\sum_m \ell_m^{1/\gamma}},
\qquad \ell\mapsto\ell^{1/\gamma},
\qquad \ell_t^{(\gamma)}=\E_{p(\cdot\mid z_t)}\big[\ell^{1/\gamma}\big].
\end{equation}
The target keeps the form of \eqref{eq:guided-score} and the weight vector is deformed instead:
$\gamma\to0$ collapses $h$ onto the best sampled candidate, while $\gamma\to\infty$ flattens it to
uniform and recovers the untilted $\zhat$.

\section{CDCD}
\label{app:sec-cdcd}
\begin{table}[t]
\centering
\small
\setlength{\tabcolsep}{4pt}
\caption{What each backbone supplies to the guided score~\eqref{eq:score-split}. Quantities (i) and (ii) are backbone-specific; the guidance term $\nabla_{z_t}\log\ell_t$ is estimated from the draws of (ii) (\secref{sec:method-estimator}).}
\label{tab:easy-hard}
\begin{tabular}{@{}lp{0.35\textwidth}p{0.35\textwidth}@{}}
\toprule
 & \textbf{CDCD} & \textbf{SeqFlow} \\
\midrule
noised variable & embedding $x$ & flow latent $g_\theta^{-1}(x)$ \\[2pt]
schedule $(\alpha_t,\sigma_t)$ & $(1,\,t)$, VE & $\big(\rho_t,\,\sqrt{1-\rho_t^2}\big)$, VP~\rev{\eqref{eq:seqflow-kernel}} \\[2pt]
(i) $\nabla_{z_t}\log p_t$ & learned~\eqref{eq:score-interp} & $-z_t$, exact \\[2pt]
(ii) draws from $p(\cdot\mid z_t)$ & mean-field $\tilde p$~\eqref{eq:mean-field}, one network pass & exact Gaussian~\eqref{eq:latent-posterior}, $N$ passes of $g_\theta$ \\[2pt]
$\nabla_{z_t}\log\ell_t$ & SNIS~\eqref{eq:snis} or REINFORCE~\eqref{eq:reinforce} & SNIS~\eqref{eq:snis} \\[2pt]
noise end & $\pstar_T\approx\mathcal{N}(0,T^2 I)$ & $\pi^\star_T\approx\mathcal{N}(0,I)$ \\[2pt]
anchor scale $\rev{\sigma}$ & $\rev{\sigma}=T$ & $\rev{\sigma}=1$ \\
\bottomrule
\end{tabular}
\end{table}

\algref{alg:cdcd-guided} covers the guided-sampling step when the backbone is a diffusion model.

\begin{algorithm}[t]
\caption{\textsc{GuidedSample} --- CDCD. Instantiates \secref{sec:method-two} at
$(\alpha_t,\sigma_t)=(1,t)$. CDCD's code was never released, so we adapted a
public reimplementation \citep{sacco2024textdiffusion}.}
\label{alg:cdcd-guided}
\begin{algorithmic}[1]
\REQUIRE $z_{\text{init}}$; frozen $\mathrm{net}_\theta$ and embedding table $E$;
         $\acq$, $\gamma$, $\lambda$, $N$, $\dataset$
\STATE $z_T \leftarrow z_{\text{init}}$
\FOR{$t \in [T,\dots,0]$, \ $\delta = t_{\text{curr}} - t_{\text{next}}$}
  \STATE $\mathrm{logits},\ \zhat \leftarrow \mathrm{net}_\theta(z_t,t)$
         \COMMENT{\eqref{eq:score-interp}}
  \STATE $w^{(n)} \sim p_\theta(\cdot \mid z_t,t)$, \ $n=1..N$
         \COMMENT{Eq.~(\plaineqref{eq:cdcd-posterior})}
  \STATE $h_n \leftarrow \softmax_n\big(\tfrac1\gamma \log \acq(w^{(n)};\dataset)\big)$
         \COMMENT{Eq.~(\plaineqref{eq:power_tilt})}
  \STATE $\ztilt \leftarrow \sum_n h_n\,E(w^{(n)})$
         \COMMENT{\propref{prop:snis} or \propref{prop:reinforce}}
  \STATE $z_{t-\delta} \leftarrow z_t + \delta\,t\,
         \dfrac{\zhat + \lambda(\ztilt - \zhat) - z_t}{t^{2}}$
        \COMMENT{Eq.~(\plaineqref{eq:linear_tilt})}
\ENDFOR
\RETURN $z_0$
\end{algorithmic}
\end{algorithm}

\paragraph{Instantiation.} CDCD's forward process \eqref{eq:cdcd-forward} is the kernel
\eqref{eq:kernel-x} of \defref{def:p_t} at
$(\alpha_t,\sigma_t)=(1,t)$, so $q_t(z_t\mid x)=\mathcal N(z_t;x,t^{2}\rev{I})$ with $x=E(w)$. \propref{prop:guidance-mean-swap} then
reads
\begin{equation}
\label{eq:cdcd-guided-score}
\nabla_{z_t}\log\pstar_t(z_t) = \frac{1}{t^{2}} \big[( \zhat - z_t) + (\ztilt - \zhat)\big] =\frac{\ztilt-z_t}{t^{2}},
\end{equation}
whose first term is CDCD's own score \eqref{eq:score-interp}: guidance replaces $\zhat$ by $\ztilt$
and leaves the sampler otherwise untouched. The sampler integrates \eqref{eq:cdcd-guided-score} by an
Euler step (\algref{alg:cdcd-guided}, line 7), which carries a \rev{discretization} error per
step --- a cost SeqFlow does not pay (\appref{app:sec-seqflow}).

With $x=E(w)$, \defref{def:tilted-estimate} reads
\begin{equation}
\label{eq:xtilt-sum}
\ztilt=\frac{ \sum_{w \in \vocab^L}p_\theta(w \mid z_t)\, \ell(w)\,E(w)}
             { \sum_{w \in \vocab^L}p_\theta(w \mid z_t)\, \ell(w)} .
\end{equation}
Both sums range over $\vocab^L$, and neither decomposes: $\ell$ is a scalar function of the whole
decoded molecule, $\ell\big((w_1,\dots,w_L)\big)$, so it couples all $L$ positions. Note the contrast
with $\zhat$, which is an expectation of a quantity that is additive across positions and is therefore
determined by the marginals \eqref{eq:cdcd-posterior} alone; $\ztilt$ is not, and must be estimated by
\propref{prop:snis}.

\paragraph{Sampling the denoising posterior.} This is where CDCD incurs its one
structural approximation, and it is worth being precise about whose approximation it is.
\propref{prop:snis} needs draws from $p(\cdot\mid z_t)$. For a diffusion model, drawing from the true
denoising posterior means running the reverse process from $z_t$ all the way to the data end, once per
sample: at $N$ samples per guidance step and $K$ steps per trajectory, that is $NK$ rollouts, which is
prohibitive. What is available for free is \eqref{eq:cdcd-posterior}, the per-position marginals, since
they come from the same forward pass that already produced $\zhat$ and drawing from a \rev{factorized}
categorical costs nothing. We therefore adopt the mean-field approximation
\begin{equation}
\label{eq:mean-field}
p(x\mid z_t)\ \approx\ \tilde p(x\mid z_t)\triangleq\textstyle\prod_{i=1}^{L} p_\theta(w_i\mid z_t,t),
\qquad x=E(w),
\end{equation}
and sample positions independently.

Two things should be said about it. First, it is \emph{our} approximation, introduced in the
estimator, not an assumption CDCD makes: the model never represents the joint because the score does
not require it, $\zhat$ being determined by the marginals alone. Second, it is empirically benign ---
we compared \eqref{eq:mean-field} against exact rollouts of the reverse process and found no
degradation in \rev{optimization} performance, at a small fraction of the cost (\tabref{tab:mean-field-ab}). The SeqFlow version needs neither: its posterior \eqref{eq:seqflow-posterior} is exactly Gaussian and fully correlated, so a draw from it is a draw from the true posterior.
\begin{table}[t]
\centering
\small
\caption{The mean-field approximation \eqref{eq:mean-field} against rollouts of the
reverse process. $K$ is the number of inner denoising steps used to draw each of the
$N$ SNIS samples: $K=0$ is the \rev{factorized} draw from the per-position marginals
\eqref{eq:cdcd-posterior}, i.e.\ \eqref{eq:mean-field}, and is what \algref{alg:cdcd-guided}
uses; $K>0$ replaces it by a $K$-step ancestral rollout, which couples the positions through
the network and resolves the joint. Final objective value after $300$ BO epochs on six
GuacaMol/TDC objectives, mean $\pm$ sd over $3$ seeds; higher is better. $\Delta$ is
$K{=}32$ minus $K{=}0$, and the last column expresses it in units of the pooled
seed-to-seed sd. At $K=32$ --- beyond the point where proposal fidelity saturates, and at
$\approx 24\times$ the wall clock ($8.77$ against $0.36$ h per arm) --- no objective improves
beyond the seed spread and one degrades.}
\label{tab:mean-field-ab}
\begin{tabular}{@{}lcccr@{}}
\toprule
objective & $K=0$ & $K=32$ & $\Delta$ & $\Delta/\mathrm{sd}$ \\
\midrule
med2 & $0.1637 \pm 0.0193$ & $0.1694 \pm 0.0243$ & $+0.0058$ & $+0.26$ \\
osmb & $0.7094 \pm 0.0578$ & $0.7055 \pm 0.0611$ & $-0.0039$ & $-0.07$ \\
adip & $0.4335 \pm 0.0297$ & $0.4483 \pm 0.0029$ & $+0.0148$ & $+0.70$ \\
zale & $0.3811 \pm 0.0935$ & $0.2842 \pm 0.0612$ & $-0.0969$ & $-1.23$ \\
rano & $0.7503 \pm 0.0361$ & $0.7088 \pm 0.0570$ & $-0.0416$ & $-0.87$ \\
qed  & $0.8992 \pm 0.0370$ & $0.9106 \pm 0.0387$ & $+0.0114$ & $+0.30$ \\
\bottomrule
\end{tabular}
\end{table}

\paragraph{REINFORCE estimator.}
\propref{prop:snis} estimates $\ztilt$ from draws of a posterior that is already available in closed
form, so the conditional score is analytic and Monte Carlo is needed only for the sequence-level
expectation of $\ell$. That estimator works directly on \defref{def:tilted-estimate}. A second,
more direct route starts from the observation that \propref{prop:guidance-mean-swap} can be rearranged
as $\ztilt = \zhat + t^{2}\,\nabla_{z_t}\log\ell_t(z_t)$. Since $\zhat$ is available from the forward
pass of the model, the only remaining obstacle is $\nabla_{z_t}\log\ell_t(z_t)$.

\begin{proposition}[REINFORCE estimator]
    \label{prop:reinforce}
    Given $\ztilt = \zhat + t^{2}\, \nabla_{z_t} \log \ell_t(z_t)$, its estimate may be written as
    \begin{equation}
    \label{eq:reinforce}
        \ztilt \approx \zhat + t^{2}\sum_{n=1}^{N} h_n\,\nabla_{z_t}\log p_\theta(w^{(n)}\mid z_t),
        \qquad h_n = \frac{\ell(w^{(n)})}{\sum_{m}\ell(w^{(m)})} .
    \end{equation}
\end{proposition}

\begin{proof}
    By \eqref{eq:tilted-marginal}, $\ell_t(z_t)=\int \ell(\tau(x))\, p(x \mid z_t)\,\mathrm{d}x = \sum_{w \in \vocab^L} \ell(w)\,p_\theta(w\mid z_t)$. Applying the log-derivative trick, i.e.\ multiplying and dividing by $p_\theta$,
    $$
    \nabla_{z_t}\ell_t
    =\sum_{w \in \vocab^L}\ell(w)\, p_\theta \frac{\nabla_{z_t}p_\theta}{p_\theta}
    =\E_{p_\theta}\big[\ell\, \nabla_{z_t} \log p_\theta \big]
    $$
    $$
    \nabla_{z_t} \log \ell_t = \frac{\E_{p_\theta} \big[ \ell\, \nabla_{z_t} \log p_\theta\big]} {\E_{p_\theta} \big[\ell \big]} =
    \frac{\displaystyle \sum_{w \in \vocab^L} \nabla_{z_t}\log p_\theta(w\mid z_t)\,\ell(w)\,p_\theta(w\mid z_t)}{\displaystyle\sum_{w \in \vocab^L}p_\theta(w\mid z_t)\,\ell(w)} .
    $$
    This is again a ratio of two expectations under the same $p_\theta$, with the same intractable denominator $\ell_t$, so Monte Carlo applies as in \propref{prop:snis}:
    $$
    \nabla_{z_t}\log\ell_t \approx \frac{\frac1N \sum_{n=1}^{N}\ell_n\, \nabla_{z_t}\log p_\theta(w^{(n)} \mid z_t)}
    {\frac1N \sum_{m=1}^{N}\ell_m}
    = \sum_{n=1}^{N}\underbrace{\frac{\ell_n}{\sum_{m}\ell_m}}_{h_n}\nabla_{z_t}\log p_\theta(w^{(n)} \mid z_t) ,
        \qquad \sum_n h_n = 1 .
    $$
    Substituting into $\ztilt= \zhat+t^{2}\,\nabla_{z_t}\log\ell_t(z_t)$ gives \eqref{eq:reinforce}.
\end{proof}

The per-sample score required by \eqref{eq:reinforce} is itself available in closed form.

\paragraph{Cost.} Both $\zhat$ and the categorical $p_\theta(w\mid z_t,t)$ come from a single forward
pass of $\mathrm{net}_\theta$, so the $N$ draws of \propref{prop:snis} are sampled from an object
already in hand and cost no additional network evaluations. Each step is therefore one network
evaluation, $N$ decodes, and one GP prediction over $N$ molecules.

\subsection{Training CDCD}
\label{app:sec-cdcd_training}
No pretrained diffusion model was available off the shelf, so we trained one from scratch.
\tabref{tab:hyp-config-cdcd} lists the hyperparameter configuration we used. \citet{dieleman2022continuous} never released the original code, so we started from a public reimplementation \citep{sacco2024textdiffusion} and adapted it to our needs.

\begin{table}[t]
  \centering
  \caption{Hyperparameters of the CDCD generator.}
  \label{tab:hyp-config-cdcd}
  \small
  \begin{tabular}{@{}l r@{}}
    \toprule
    \textbf{Hyperparameter} & \textbf{Value} \\
    \midrule
    \multicolumn{2}{@{}l}{\textit{Architecture}} \\
    \quad Embedding dimension           & $16$  \\
    \quad Hidden dimension              & $512$ \\
    \quad Query/key/value dimension     & $512$ \\
    \quad Time embedding dimension      & $16$  \\
    \quad Layers                        & $8$   \\
    \quad Attention heads               & $8$   \\
    \addlinespace
    \multicolumn{2}{@{}l}{\textit{Conditioning}} \\
    \quad Self-conditioning probability & $0.4$ \\
    \quad Mask-conditioning probability & $0.1$ \\
    \quad Masking probability           & $1.0$ \\
    \addlinespace
    \multicolumn{2}{@{}l}{\textit{Loss}} \\
    \quad Objective                     & Focal \citep{focal} \\
    \quad \texttt{<nop>} token weight   & $1.0$ \\
    \quad \texttt{<eos>} token weight   & $1.0$ \\
    \quad \texttt{<pad>} token weight   & $1.0$ \\
    \addlinespace
    \multicolumn{2}{@{}l}{\textit{Noise schedule}} \\
    \quad Minimum time                  & $0.01$  \\
    \quad Maximum time                  & $200.0$ \\
    \quad Location                      & $3.0$   \\
    \quad Scale                         & $2.0$   \\
    \quad Height                        & $6.0$   \\
    \quad Offset                        & $-1.0$  \\
    \quad Update frequency              & $100$   \\
    \addlinespace
    \multicolumn{2}{@{}l}{\textit{\rev{Optimization}}} \\
    \quad Batch size                    & $512$              \\
    \quad Learning rate                 & $5 \times 10^{-6}$ \\
    \quad Maximum epochs                & $30$               \\
    \quad Maximum steps                 & $450{,}000$        \\
    \quad Gradient clipping             & $1.0$              \\
    \quad Gradient accumulation steps   & $1$                \\
    \addlinespace
    \multicolumn{2}{@{}l}{\textit{Sequence and sampling}} \\
    \quad Maximum sequence length       & $110$ \\
    \quad Integration steps             & $200$ \\
    \bottomrule
  \end{tabular}
\end{table}

\section{SeqFlow}
\label{app:sec-seqflow}
\algref{alg:seqflow-guided} covers the guided-sampling step when the backbone is a normalizing flow.

\begin{algorithm}[t]
\caption{\textsc{GuidedSample} --- SeqFlow. Instantiates \secref{sec:method-two} in the flow latent with
the variance-preserving schedule \rev{\eqref{eq:seqflow-kernel}}.}
\label{alg:seqflow-guided}
\begin{algorithmic}[1]
\REQUIRE $z_{\text{init}}$; frozen flow $g_\theta$ and readout $\tau$;
         schedule $\rho_T<\dots<\rho_0$; $\acq$, $\gamma$, $\lambda$, $N$, $\dataset$
\STATE $z_T \leftarrow z_{\text{init}}$
\FOR{$t = T,\dots,0$, \ $t'$ the next level}
  \STATE $\zhat \leftarrow \rho_t\,z_t$
         \COMMENT{closed form; no network eval, \eqref{eq:seqflow-posterior}}
  \STATE $z^{(n)} \sim \mathcal N\big(\zhat,(1-\rho_t^{2})\rev{I}\big)$, \ $n=1..N$
         \COMMENT{exact posterior}
  \STATE $h_n \leftarrow \softmax_n\big(\tfrac1\gamma \log \acq(\tau(g_\theta(z^{(n)}));\dataset)\big)$
         \COMMENT{Eq.~(\plaineqref{eq:power_tilt})}
  \STATE $\ztilt \leftarrow \sum_n h_n\,z^{(n)}$
         \COMMENT{\propref{prop:snis}}
  \STATE $z_{t'} \leftarrow \rho_{t'}\big(\zhat + \lambda(\ztilt - \zhat)\big)
         + \sqrt{1-\rho_{t'}^{2}}\,\xi$, \ $\xi \sim \mathcal N(0,\rev{I})$
\ENDFOR
\RETURN $z_0$, decoded as $\tau(g_\theta(z_0))$
\end{algorithmic}
\end{algorithm}

\paragraph{Instantiation.} SeqFlow supplies no noising process of its own, so the kernel of
\defref{def:p_t} is imposed rather than inherited. We work in the flow latent
$z=g_\theta^{-1}(x)$, whose law is exactly
$\mathcal N(0,\rev{I})$ while the tilt is the pullback $\ell\big(\tau(g_\theta(z))\big)$, \eqref{eq:push-forward}. Set
\begin{equation}
\label{eq:seqflow-kernel}
q_t(z_t\mid z)=\mathcal N\big(z_t;\,\rho_t\,z,\,(1-\rho_t^{2})\rev{I}\big),
\qquad \rho_t\in(0,1],
\end{equation}
i.e.\ $(\alpha_t,\sigma_t)=\big(\rho_t,\sqrt{1-\rho_t^{2}}\big)$ in \eqref{eq:kernel-z}. Two closed
forms follow. Since the latent law is $\mathcal N(0,\rev{I})$, the untilted marginal is
\begin{equation}
\label{eq:seqflow-marginal}
p_t(z_t)=\int \mathcal N\big(z_t;\rho_tz,(1-\rho_t^{2})\rev{I}\big)\,\mathcal N(z;0,\rev{I})\,dz
=\mathcal N(z_t;0,\rev{I}),
\end{equation}
so the bridge is variance preserving and every level carries the same law as the flow's own base;
and by Gaussian conjugacy the posterior of \propref{prop:tilted-marginal} is exact,
\begin{equation}
\label{eq:seqflow-posterior}
p(z\mid z_t)=\mathcal N\big(z;\,\rho_t\,z_t,\,(1-\rho_t^{2})\rev{I}\big),
\qquad \zhat = \rho_t\,z_t ,
\end{equation}
which is \eqref{eq:latent-posterior} \rev{with $\alpha_t=\rho_t$}. Neither requires a network evaluation.
\propref{prop:guidance-mean-swap} then reads
\begin{equation}
\label{eq:seqflow-guided-score}
\nabla_{z_t}\log\pstar_t(z_t)=\frac{\rho_t\,\ztilt-z_t}{1-\rho_t^{2}} .
\end{equation}
Nothing about $g_\theta$ enters any of \eqref{eq:seqflow-marginal}--\eqref{eq:seqflow-guided-score}: the
learned content of the flow is confined to the tilt, where it is evaluated pointwise by
\propref{prop:snis} and never differentiated. This is why no analogue of \propref{prop:reinforce} is
needed here --- there is no learned posterior whose score would have to be estimated.

\paragraph{Why tilting the latent is legitimate.} \defref{def:p_x0_tild} is stated over the object
that $\ell$ scores, whereas \eqref{eq:seqflow-kernel} tilts $z$. The two agree by the argument of
\appref{proof:lifting}, applied with $\tau$ replaced by the composite readout
$\tau\circ g_\theta$: writing $\Pi$ for the distribution over molecules induced by
$\pi^\star(z)\propto\mathcal N(z;0,\rev{I})\,\ell(\tau(g_\theta(z)))$ of \eqref{eq:push-forward}, and using that $\ell$ is
constant on each cell $A_w$ of \eqref{eq:seqflow-pushforward},
\begin{equation}
\label{eq:pushforward-commutes}
\Pi(w)=\frac1Z\int_{A_w}\mathcal N(z;0,\rev{I})\,\ell\big(\tau(g_\theta(z))\big)\,dz
=\frac{\ell(w)}{Z}\int_{A_w}\mathcal N(z;0,\rev{I})\,dz
=\frac{P_\theta(w)\,\ell(w)}{Z},
\end{equation}
the intended target. Tilting the latent and decoding is therefore the same distribution over
molecules as tilting the molecules directly; the constancy of $\ell$ on cells is what licenses pulling
it out of the integral.


\paragraph{What remains approximate.} Three things, none of them the posterior. First, the $N$-sample estimate of $\ztilt$ in \propref{prop:snis}, which is common to both backbones. Second, \algref{alg:seqflow-guided} substitutes the tilted posterior \emph{mean} $\ztilt$ for a \emph{draw} from $\pstar(\cdot\mid z_t)$, which is under-dispersed and mode-seeking; resampling a single $z^{(n)}$ with probability $h_n$ instead recovers a transition that is asymptotically exact in $N$, at the cost of variance. Third, the rollout is \rev{initialized} at $z_T\sim p_T=\mathcal N(0,\rev{I})$ rather than at $\pi^\star_T$, which by \propref{prop:tilted-marginal} differs from $p_T$ by the tilt $\ell_T$; the discrepancy vanishes as $\rho_T\to0$ and is the price of $\pi^\star_T$ not being samplable.

\paragraph{Cost.} Each step draws $N$ latents from \eqref{eq:seqflow-posterior} and decodes each through $g_\theta$ and $\tau$, so the flow is evaluated $N$ times per step. CDCD \rev{amortizes} a single forward pass across all $N$ draws, since \eqref{eq:cdcd-posterior} is already available from it. This is the price of the exact posterior: the flow has no \rev{amortized} categorical to sample from, and every draw must be pushed through the network to be scored. The decode is latency-bound rather than throughput-bound, however, so $N$ is nearly free and the step count is the binding constraint --- the opposite of CDCD's economics.

\subsection{Training SeqFlow}
We used the pretrained model available in the authors' repository. For details of how it was trained and which hyperparameter configuration was used, we refer to \citet{lee2025nfbo}.
\section{AFGG Algorithm}
\label{sec:afgg_algo}

\algref{alg:afgg} summarizes the complete AFGG optimization loop.
At each BO iteration, the surrogate is updated in data space, the guided sampler is initialized around the current incumbent using the trust region above, and the resulting candidate is decoded and evaluated.
The overall procedure is shared across both generative backbones; they differ only in the continuous variable in which the bridge and guidance are applied.
For CDCD, this variable is the token-embedding representation, whereas for SeqFlow it is the Gaussian flow latent.
The complete guided samplers are given in \appref{app:sec-cdcd} and \appref{app:sec-seqflow} (\algref{alg:cdcd-guided} and \algref{alg:seqflow-guided}), and \tabref{tab:easy-hard} summarizes the correspondence between the two instantiations.

\begin{algorithm}[t]
\caption{AFGG: full BO loop}
\label{alg:afgg}
\begin{algorithmic}[1]
\REQUIRE frozen generator; acquisition $\acq$; temperature $\gamma$; guidance scale $\lambda$;
inner samples $N$; batch size $q$; oracle budget $B$; initial dataset $\dataset_0$;
trust-region coefficient $\rho$; anchor scale $\sigma$
\STATE $\dataset \leftarrow \dataset_0$
\FOR{$r = 1,2,\dots,\lfloor B/q\rfloor$}
  \STATE fit GP on $\dataset$ \COMMENT{in design space, \secref{sec:bo}}
  \STATE $z^{\star} \leftarrow$ anchor from the incumbent
  \STATE $z^{(1:q)}_{\text{init}} \leftarrow \rho\,z^{\star} + \sqrt{1-\rho^{2}}\,\varepsilon$,
         \quad $\varepsilon \sim \mathcal N(0,\sigma^{2}\rev{I})$ \COMMENT{\eqref{eq:trust-region}}
  \STATE $z^{(1:q)} \leftarrow \textsc{GuidedSample} \big( z^{(1:q)}_{\text{init}} \big)$
         \COMMENT{\algref{alg:cdcd-guided} (CDCD) or \algref{alg:seqflow-guided} (SeqFlow)}
  \STATE $w^{(i)} \leftarrow \tau(z^{(i)})$ (CDCD) or $\tau\big(g_\theta(z^{(i)})\big)$ (SeqFlow), \ $i=1..q$
         \COMMENT{decode to design space}
  \STATE $y^{(i)} \leftarrow$ evaluate oracle on $w^{(i)}$, \ $i=1..q$
  \STATE $\dataset \leftarrow \dataset \cup \{(w^{(i)},y^{(i)})\}_{i=1}^{q}$
\ENDFOR
\RETURN $\argmax_{(w,y)\in\dataset} y$
\end{algorithmic}
\end{algorithm}

\section{Additional \rev{Material}}

\subsection{Additional \rev{Results} on GuacaMol}
\label{sec:additional_results}
This section reports the full benchmark behind \secref{sec:experiments}: all 25 PMO objectives, in the sequential regime and in both batch regimes, as final scores and as complete trajectories. Objective names are abbreviated as in \tabref{tab:obj-abbrev}.

\begin{table}[t]
\centering
\small
\begin{tabular}{@{}ll@{\hspace{2em}}ll@{}}
\toprule
Abbreviation & PMO objective & Abbreviation & PMO objective \\
\midrule
\texttt{albu\_sim} & \texttt{albuterol\_similarity} & \texttt{osim\_mpo} & \texttt{\rev{osimertinib}\_mpo} \\
\texttt{amlo\_mpo} & \texttt{amlodipine\_mpo} & \texttt{peri\_mpo} & \texttt{perindopril\_mpo} \\
\texttt{cele\_red} & \texttt{celecoxib\_rediscovery} & \texttt{rano\_mpo} & \texttt{ranolazine\_mpo} \\
\texttt{deco\_hop} & \texttt{deco\_hop} & \texttt{logp} & \texttt{rdkit\_logp} \\
\texttt{drd2} & \texttt{drd2\_docking} & \texttt{qed} & \texttt{rdkit\_qed} \\
\texttt{fexo\_mpo} & \texttt{fexofenadine\_mpo} & \texttt{sa} & \texttt{sa\_tdc} \\
\texttt{gsk3b} & \texttt{gsk3\_beta} & \texttt{scaf\_hop} & \texttt{scaffold\_hop} \\
\texttt{iso\_c7} & \texttt{isomer\_c7h8n2o2} & \texttt{sita\_mpo} & \texttt{sitagliptin\_mpo} \\
\texttt{iso\_c9} & \texttt{isomer\_c9h10n2o2pf2cl} & \texttt{thio\_red} & \texttt{thiothixene\_rediscovery} \\
\texttt{jnk3} & \texttt{jnk3} & \texttt{trog\_red} & \texttt{troglitazone\_rediscovery} \\
\texttt{med\_1} & \texttt{median\_1} & \texttt{vals\_sma} & \texttt{valsartan\_smarts} \\
\texttt{med\_2} & \texttt{median\_2} & \texttt{zale\_mpo} & \texttt{zaleplon\_mpo} \\
\texttt{mest\_sim} & \texttt{mestranol\_similarity} &  &  \\
\bottomrule
\end{tabular}
\caption{Abbreviations of the 25 PMO objectives, used in all tables.}
\label{tab:obj-abbrev}
\end{table}

\rev{All available generative-model checkpoints of the baselines were trained on GuacaMol; COWBOYS provides no checkpoint, so we retrain it on GuacaMol from the released code.}

One accounting point applies throughout. CoBO \citep{lee2023cobo} and LOL-BO \citep{maus2022local} spend oracle calls realigning a latent with the molecule it decodes to, and their released code does not charge those calls to the budget. Following \citet{chu2024inversion} we modified both to count them, so every method below is compared at the same number of oracle evaluations.

\figref{fig:25-obj-seq} and \rev{\tabref{tab:afgg-all25-vs-benchmarks_full}} give the sequential regime ($q=1$, 10 initial points, 300 oracle calls) against all eleven baselines, of which \tabref{tab:afgg-all25-vs-benchmarks_partial} in the main text shows five. Each cell is judged against the opposing group, so a bold AFGG cell beats every baseline on that objective and a bold baseline cell beats both AFGG columns; green marks a lead exceeding the cell's own standard deviation. Two objectives are worth reading separately from the rest. \texttt{logp} and \texttt{sa} are not scaled to be in $[0,1]$, and reward molecules that are large or synthetically implausible, so a high score there measures how far a method will leave the training distribution rather than how well it \rev{optimizes}; the LSBO baselines lead on both. 

\begin{figure}[t]
    \centering
    \includegraphics[width=\textwidth]{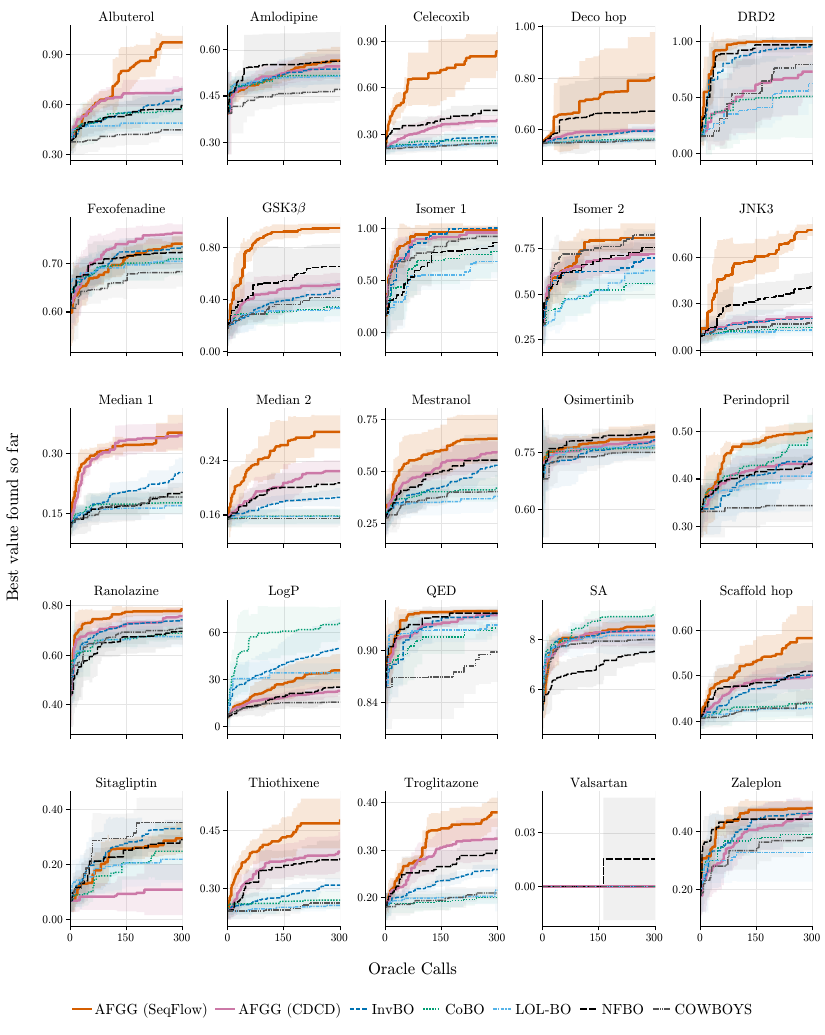}
    \caption{\textbf{Sequential setting}: full trajectories on all 25 PMO objectives (10 initial  points, 300 oracle calls, 1 point per BO iteration).}
    \label{fig:25-obj-seq}
\end{figure}

\begin{table}[t]
\centering
\small
\resizebox{\textwidth}{!}{%
\begin{tabular}{lrrrrrrrrrrrrr}
\toprule
 & \multicolumn{11}{c}{baselines} & \multicolumn{2}{c}{AFGG} \\
\cmidrule(lr){2-12}\cmidrule(lr){13-14}
Objective & CoBO & InvBO & LOL-BO & NFBO & CMA-ES & COWBOYS & DE & GA & HC & RLBO & TuRBO & CDCD & SeqFlow \\
\midrule
\texttt{albu\_sim} & 0.593 $\pm$ 0.111 & 0.629 $\pm$ 0.079 & 0.509 $\pm$ 0.065 & 0.593 $\pm$ 0.119 & 0.500 $\pm$ 0.030 & 0.448 $\pm$ 0.037 & 0.488 $\pm$ 0.035 & 0.385 $\pm$ 0.074 & 0.454 $\pm$ 0.070 & 0.483 $\pm$ 0.091 & 0.475 $\pm$ 0.047 & \textbf{0.691 $\pm$ 0.073} & \cellcolor[HTML]{D6F0D6}\textbf{0.975 $\pm$ 0.038} \\
\texttt{amlo\_mpo} & 0.523 $\pm$ 0.012 & 0.538 $\pm$ 0.050 & 0.516 $\pm$ 0.022 & 0.563 $\pm$ 0.095 & 0.459 $\pm$ 0.038 & 0.472 $\pm$ 0.038 & 0.438 $\pm$ 0.143 & 0.419 $\pm$ 0.082 & 0.460 $\pm$ 0.048 & 0.491 $\pm$ 0.030 & 0.475 $\pm$ 0.024 & 0.546 $\pm$ 0.039 & \textbf{0.565 $\pm$ 0.045} \\
\texttt{cele\_red} & 0.262 $\pm$ 0.041 & 0.286 $\pm$ 0.068 & 0.243 $\pm$ 0.035 & 0.455 $\pm$ 0.035 & 0.227 $\pm$ 0.032 & 0.243 $\pm$ 0.020 & 0.216 $\pm$ 0.036 & 0.210 $\pm$ 0.040 & 0.228 $\pm$ 0.032 & 0.224 $\pm$ 0.031 & 0.221 $\pm$ 0.033 & 0.392 $\pm$ 0.024 & \cellcolor[HTML]{D6F0D6}\textbf{0.836 $\pm$ 0.125} \\
\texttt{deco\_hop} & 0.565 $\pm$ 0.008 & 0.597 $\pm$ 0.019 & 0.560 $\pm$ 0.009 & 0.673 $\pm$ 0.147 & 0.563 $\pm$ 0.010 & 0.559 $\pm$ 0.008 & 0.571 $\pm$ 0.018 & 0.549 $\pm$ 0.015 & 0.553 $\pm$ 0.012 & 0.556 $\pm$ 0.015 & 0.552 $\pm$ 0.012 & 0.599 $\pm$ 0.010 & \textbf{0.803 $\pm$ 0.175} \\
\texttt{drd2} & 0.510 $\pm$ 0.370 & 0.963 $\pm$ 0.037 & 0.626 $\pm$ 0.184 & 0.971 $\pm$ 0.066 & 0.184 $\pm$ 0.068 & 0.796 $\pm$ 0.243 & 0.862 $\pm$ 0.156 & 0.154 $\pm$ 0.051 & 0.183 $\pm$ 0.045 & 0.380 $\pm$ 0.246 & 0.556 $\pm$ 0.337 & 0.730 $\pm$ 0.226 & \cellcolor[HTML]{D6F0D6}\textbf{1.000 $\pm$ 0.000} \\
\texttt{fexo\_mpo} & 0.710 $\pm$ 0.034 & 0.734 $\pm$ 0.027 & 0.716 $\pm$ 0.021 & 0.723 $\pm$ 0.032 & 0.672 $\pm$ 0.014 & 0.683 $\pm$ 0.031 & 0.724 $\pm$ 0.045 & 0.611 $\pm$ 0.063 & 0.680 $\pm$ 0.033 & 0.682 $\pm$ 0.034 & 0.688 $\pm$ 0.026 & \cellcolor[HTML]{D6F0D6}\textbf{0.763 $\pm$ 0.017} & \textbf{0.741 $\pm$ 0.016} \\
\texttt{gsk3b} & 0.362 $\pm$ 0.069 & 0.500 $\pm$ 0.042 & 0.336 $\pm$ 0.103 & 0.656 $\pm$ 0.174 & 0.371 $\pm$ 0.042 & 0.418 $\pm$ 0.093 & 0.478 $\pm$ 0.108 & 0.248 $\pm$ 0.081 & 0.320 $\pm$ 0.047 & 0.343 $\pm$ 0.084 & 0.328 $\pm$ 0.047 & 0.518 $\pm$ 0.088 &\cellcolor[HTML]{D6F0D6}\textbf{0.950 $\pm$ 0.037}  \\
\texttt{iso\_c7} & 0.875 $\pm$ 0.145 & \cellcolor[HTML]{D6F0D6}\textbf{1.000 $\pm$ 0.000} & 0.760 $\pm$ 0.036 & 0.860 $\pm$ 0.092 & 0.972 $\pm$ 0.051 & 0.919 $\pm$ 0.046 & 0.808 $\pm$ 0.051 & 0.808 $\pm$ 0.051 & 0.895 $\pm$ 0.090 & 0.898 $\pm$ 0.072 & 0.898 $\pm$ 0.072 & 0.957 $\pm$ 0.049 & 0.981 $\pm$ 0.043 \\
\texttt{iso\_c9} & 0.577 $\pm$ 0.050 & 0.699 $\pm$ 0.147 & 0.630 $\pm$ 0.082 & 0.758 $\pm$ 0.024 & 0.751 $\pm$ 0.074 & \cellcolor[HTML]{D6F0D6}\textbf{0.834 $\pm$ 0.020} & 0.672 $\pm$ 0.146 & 0.643 $\pm$ 0.103 & 0.687 $\pm$ 0.106 & 0.714 $\pm$ 0.035 & 0.778 $\pm$ 0.049 & 0.722 $\pm$ 0.079 & 0.809 $\pm$ 0.081 \\
\texttt{jnk3} & 0.148 $\pm$ 0.031 & 0.206 $\pm$ 0.063 & 0.144 $\pm$ 0.035 & 0.412 $\pm$ 0.095 & 0.150 $\pm$ 0.025 & 0.176 $\pm$ 0.032 & 0.190 $\pm$ 0.053 & 0.142 $\pm$ 0.029 & 0.144 $\pm$ 0.025 & 0.150 $\pm$ 0.049 & 0.158 $\pm$ 0.029 & 0.216 $\pm$ 0.033 & \cellcolor[HTML]{D6F0D6}\textbf{0.778 $\pm$ 0.038}\\
\texttt{med\_1} & 0.178 $\pm$ 0.025 & 0.253 $\pm$ 0.037 & 0.175 $\pm$ 0.023 & 0.203 $\pm$ 0.024 & 0.146 $\pm$ 0.012 & 0.192 $\pm$ 0.024 & 0.179 $\pm$ 0.038 & 0.152 $\pm$ 0.042 & 0.154 $\pm$ 0.015 & 0.150 $\pm$ 0.023 & 0.181 $\pm$ 0.030 & \cellcolor[HTML]{D6F0D6}\textbf{0.346 $\pm$ 0.031} & \cellcolor[HTML]{D6F0D6}\textbf{0.352 $\pm$ 0.045} \\
\texttt{med\_2} & 0.158 $\pm$ 0.011 & 0.186 $\pm$ 0.014 & 0.158 $\pm$ 0.011 & 0.207 $\pm$ 0.018 & 0.155 $\pm$ 0.013 & 0.156 $\pm$ 0.012 & 0.168 $\pm$ 0.017 & 0.155 $\pm$ 0.013 & 0.155 $\pm$ 0.013 & 0.155 $\pm$ 0.013 & 0.155 $\pm$ 0.013 & \cellcolor[HTML]{D6F0D6}\textbf{0.225 $\pm$ 0.016} & \cellcolor[HTML]{D6F0D6}\textbf{0.283 $\pm$ 0.024} \\
\texttt{mest\_sim} & 0.421 $\pm$ 0.064 & 0.531 $\pm$ 0.054 & 0.384 $\pm$ 0.059 & 0.555 $\pm$ 0.125 & 0.362 $\pm$ 0.062 & 0.404 $\pm$ 0.023 & 0.382 $\pm$ 0.104 & 0.280 $\pm$ 0.090 & 0.358 $\pm$ 0.058 & 0.372 $\pm$ 0.046 & 0.395 $\pm$ 0.057 & \textbf{0.593 $\pm$ 0.051} & \textbf{0.658 $\pm$ 0.114} \\
\texttt{osim\_mpo} & 0.761 $\pm$ 0.013 & 0.782 $\pm$ 0.024 & 0.767 $\pm$ 0.015 & \textbf{0.803 $\pm$ 0.019} & 0.750 $\pm$ 0.020 & 0.749 $\pm$ 0.023 & 0.756 $\pm$ 0.067 & 0.724 $\pm$ 0.060 & 0.744 $\pm$ 0.028 & 0.749 $\pm$ 0.027 & 0.748 $\pm$ 0.022 & 0.778 $\pm$ 0.023 & 0.789 $\pm$ 0.034 \\
\texttt{peri\_mpo} & 0.487 $\pm$ 0.049 & 0.447 $\pm$ 0.054 & 0.416 $\pm$ 0.079 & 0.434 $\pm$ 0.015 & 0.345 $\pm$ 0.053 & 0.344 $\pm$ 0.034 & 0.355 $\pm$ 0.067 & 0.332 $\pm$ 0.055 & 0.344 $\pm$ 0.034 & 0.356 $\pm$ 0.023 & 0.342 $\pm$ 0.038 & 0.432 $\pm$ 0.031 & \textbf{0.502 $\pm$ 0.017} \\
\texttt{rano\_mpo} & 0.688 $\pm$ 0.027 & 0.742 $\pm$ 0.033 & 0.684 $\pm$ 0.042 & 0.696 $\pm$ 0.053 & 0.690 $\pm$ 0.071 & 0.708 $\pm$ 0.040 & 0.729 $\pm$ 0.068 & 0.484 $\pm$ 0.081 & 0.710 $\pm$ 0.043 & 0.672 $\pm$ 0.103 & 0.728 $\pm$ 0.025 & \textbf{0.757 $\pm$ 0.034} & \cellcolor[HTML]{D6F0D6}\textbf{0.785 $\pm$ 0.007} \\
\texttt{logp} & \cellcolor[HTML]{D6F0D6}\textbf{65.866 $\pm$ 10.720} & \textbf{49.799 $\pm$ 19.126} & 34.741 $\pm$ 5.318 & 25.055 $\pm$ 8.705 & 13.285 $\pm$ 1.014 & 15.443 $\pm$ 3.623 & 7.589 $\pm$ 3.102 & 6.138 $\pm$ 0.827 & 11.939 $\pm$ 1.801 & 10.491 $\pm$ 2.039 & 12.128 $\pm$ 1.918 & 22.684 $\pm$ 2.667 & 35.734 $\pm$ 10.360 \\
\texttt{qed} & 0.928 $\pm$ 0.004 & 0.942 $\pm$ 0.006 & 0.934 $\pm$ 0.003 & 0.943 $\pm$ 0.008 & 0.888 $\pm$ 0.032 & 0.898 $\pm$ 0.034 & 0.888 $\pm$ 0.034 & 0.857 $\pm$ 0.047 & 0.889 $\pm$ 0.038 & 0.882 $\pm$ 0.045 & 0.896 $\pm$ 0.021 & 0.940 $\pm$ 0.006 & \cellcolor[HTML]{D6F0D6}\textbf{0.945 $\pm$ 0.002} \\
\texttt{sa} & \cellcolor[HTML]{D6F0D6}\textbf{9.001 $\pm$ 0.347} & 8.393 $\pm$ 0.543 & 8.214 $\pm$ 0.711 & 7.538 $\pm$ 0.452 & 8.035 $\pm$ 0.069 & 8.020 $\pm$ 0.363 & 8.492 $\pm$ 0.742 & \cellcolor[HTML]{D6F0D6}\textbf{8.853 $\pm$ 0.151} & 8.053 $\pm$ 0.225 & 8.035 $\pm$ 0.128 & 7.971 $\pm$ 0.035 & 8.349 $\pm$ 0.606 & 8.560 $\pm$ 0.220 \\
\texttt{scaf\_hop} & 0.441 $\pm$ 0.029 & 0.501 $\pm$ 0.030 & 0.430 $\pm$ 0.020 & 0.510 $\pm$ 0.062 & 0.431 $\pm$ 0.008 & 0.442 $\pm$ 0.027 & 0.441 $\pm$ 0.027 & 0.407 $\pm$ 0.023 & 0.413 $\pm$ 0.018 & 0.417 $\pm$ 0.027 & 0.427 $\pm$ 0.029 & 0.500 $\pm$ 0.028 & \cellcolor[HTML]{D6F0D6}\textbf{0.583 $\pm$ 0.070} \\
\texttt{sita\_mpo} & 0.255 $\pm$ 0.096 & \cellcolor[HTML]{D6F0D6}\textbf{0.330 $\pm$ 0.032} & 0.228 $\pm$ 0.087 & 0.290 $\pm$ 0.030 & \cellcolor[HTML]{D6F0D6}\textbf{0.367 $\pm$ 0.037} & \textbf{0.352 $\pm$ 0.090} & 0.275 $\pm$ 0.095 & 0.094 $\pm$ 0.050 & \cellcolor[HTML]{D6F0D6}\textbf{0.400 $\pm$ 0.101} & \textbf{0.319 $\pm$ 0.118} & 0.275 $\pm$ 0.063 & 0.108 $\pm$ 0.092 & 0.295 $\pm$ 0.056 \\
\texttt{thio\_red} & 0.270 $\pm$ 0.025 & 0.308 $\pm$ 0.017 & 0.255 $\pm$ 0.011 & 0.376 $\pm$ 0.053 & 0.250 $\pm$ 0.015 & 0.261 $\pm$ 0.010 & 0.272 $\pm$ 0.039 & 0.240 $\pm$ 0.021 & 0.241 $\pm$ 0.019 & 0.245 $\pm$ 0.014 & 0.250 $\pm$ 0.014 & \textbf{0.395 $\pm$ 0.051} & \cellcolor[HTML]{D6F0D6}\textbf{0.475 $\pm$ 0.060} \\
\texttt{trog\_red} & 0.205 $\pm$ 0.004 & 0.260 $\pm$ 0.010 & 0.218 $\pm$ 0.044 & 0.300 $\pm$ 0.044 & 0.206 $\pm$ 0.017 & 0.210 $\pm$ 0.014 & 0.205 $\pm$ 0.033 & 0.182 $\pm$ 0.028 & 0.194 $\pm$ 0.016 & 0.196 $\pm$ 0.016 & 0.190 $\pm$ 0.017 & \textbf{0.324 $\pm$ 0.043} & \cellcolor[HTML]{D6F0D6}\textbf{0.380 $\pm$ 0.031} \\
\texttt{vals\_sma} & 0.000 $\pm$ 0.000 & 0.000 $\pm$ 0.000 & 0.000 $\pm$ 0.000 & \textbf{0.015 $\pm$ 0.034} & 0.000 $\pm$ 0.000 & 0.000 $\pm$ 0.000 & 0.000 $\pm$ 0.000 & 0.000 $\pm$ 0.000 & 0.000 $\pm$ 0.000 & 0.000 $\pm$ 0.000 & 0.000 $\pm$ 0.000 & 0.000 $\pm$ 0.000 & 0.000 $\pm$ 0.000 \\
\texttt{zale\_mpo} & 0.404 $\pm$ 0.055 & 0.461 $\pm$ 0.047 & 0.348 $\pm$ 0.065 & 0.442 $\pm$ 0.042 & 0.390 $\pm$ 0.037 & 0.378 $\pm$ 0.044 & 0.361 $\pm$ 0.123 & 0.191 $\pm$ 0.068 & 0.365 $\pm$ 0.037 & 0.357 $\pm$ 0.083 & 0.355 $\pm$ 0.041 & \textbf{0.466 $\pm$ 0.037} & \textbf{0.480 $\pm$ 0.026} \\
\bottomrule
\end{tabular}}
\caption{Final best score on all 25 PMO objectives after 300 oracle calls (10 initial points, 1 point per BO iteration), mean $\pm$ sample standard deviation over 5 seeds. AFGG is shown in two instantiations: \textbf{CDCD} and \textbf{SeqFlow}. Every cell is judged against the opposing group: an AFGG cell against the 11 baselines, a baseline cell against the two AFGG columns. \textbf{Bold}: the cell's mean exceeds every mean in the opposing group. \colorbox[HTML]{D6F0D6}{Green}: that lead is significant, exceeding the cell's own standard deviation; bold without green therefore means ahead but within noise.}
\label{tab:afgg-all25-vs-benchmarks_full}
\end{table}

The batch regime (100 initial points, then 100 epochs of $\rev{q}=5$) admits two ways of turning a per-candidate tilt into a batch, differing in where the interaction between members enters. 
\emph{Uncoupled} generation guides each trajectory by its own $\ell=\alpha^{\mathrm{EI}}$ and lets the points ``meet'' only at selection, where the $\rev{q}$ points are chosen from the generated pool by joint qLogEI in design space. \emph{Coupled} generation uses one joint tilt per batch, $\ell$ being qLogEI evaluated on the partial batch, so each member is generated aware of those already fixed. Coupled generation is the more faithful instantiation of batch BO --- it is the one that targets the joint acquisition the batch is eventually scored by --- but it is not the one that performs better here, so we report both. \figref{fig:25-obj-batch-uncoupled} and \tabref{tab:afgg-batch-uncoupled} give the uncoupled variant, the one used in \secref{sec:experiments}; \figref{fig:25-obj-batch-coupled} and \tabref{tab:afgg-batch-coupled} the coupled one. The baseline columns are the same runs in both tables, since only the AFGG variants differ.

\begin{figure}[t]
    \centering
    \includegraphics[width=\textwidth]{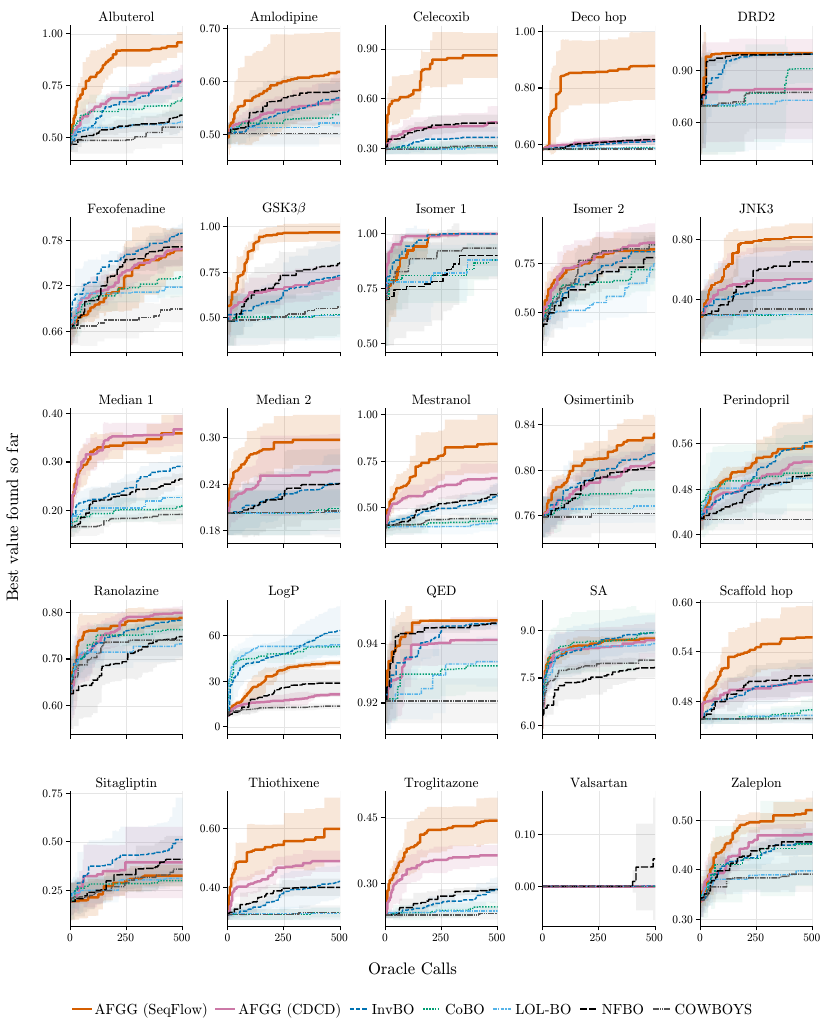}
    \caption{\textbf{Batch regime, uncoupled generation}: 500 acquisitions (100 initial points, then 100 epochs of 5), all 25 objectives. }
    \label{fig:25-obj-batch-uncoupled}
\end{figure}

\begin{table}[t]
\centering
\small
\resizebox{\textwidth}{!}{%
\begin{tabular}{lrrrrrrr}
\toprule
 & \multicolumn{5}{c}{baselines} & \multicolumn{2}{c}{AFGG} \\
\cmidrule(lr){2-6}\cmidrule(lr){7-8}
Objective & CoBO & InvBO & LOL-BO & NFBO & COWBOYS & CDCD & SeqFlow \\
\midrule
\texttt{albu\_sim} & 0.692 $\pm$ 0.095 & 0.771 $\pm$ 0.057 & 0.579 $\pm$ 0.043 & 0.609 $\pm$ 0.073 & 0.551 $\pm$ 0.081 & \textbf{0.779 $\pm$ 0.073} & \cellcolor[HTML]{D6F0D6}\textbf{0.960 $\pm$ 0.050} \\
\texttt{amlo\_mpo} & 0.540 $\pm$ 0.021 & 0.569 $\pm$ 0.029 & 0.522 $\pm$ 0.026 & 0.582 $\pm$ 0.024 & 0.501 $\pm$ 0.021 & 0.565 $\pm$ 0.052 & \textbf{0.618 $\pm$ 0.075} \\
\texttt{cele\_red} & 0.307 $\pm$ 0.035 & 0.366 $\pm$ 0.064 & 0.304 $\pm$ 0.036 & 0.452 $\pm$ 0.052 & 0.316 $\pm$ 0.052 & \textbf{0.457 $\pm$ 0.098} & \cellcolor[HTML]{D6F0D6}\textbf{0.861 $\pm$ 0.135} \\
\texttt{deco\_hop} & 0.588 $\pm$ 0.005 & 0.610 $\pm$ 0.007 & 0.586 $\pm$ 0.004 & 0.617 $\pm$ 0.018 & 0.583 $\pm$ 0.004 & 0.614 $\pm$ 0.019 & \cellcolor[HTML]{D6F0D6}\textbf{0.878 $\pm$ 0.120} \\
\texttt{drd2} & 0.930 $\pm$ 0.090 & 0.996 $\pm$ 0.006 & 0.728 $\pm$ 0.247 & 0.994 $\pm$ 0.012 & 0.774 $\pm$ 0.210 & 0.792 $\pm$ 0.288 & \cellcolor[HTML]{D6F0D6}\textbf{1.000 $\pm$ 0.000} \\
\texttt{fexo\_mpo} & 0.732 $\pm$ 0.025 & \cellcolor[HTML]{D6F0D6}\textbf{0.789 $\pm$ 0.013} & 0.718 $\pm$ 0.010 & \textbf{0.771 $\pm$ 0.025} & 0.689 $\pm$ 0.017 & 0.768 $\pm$ 0.022 & 0.767 $\pm$ 0.028 \\
\texttt{gsk3b} & 0.516 $\pm$ 0.125 & 0.730 $\pm$ 0.174 & 0.512 $\pm$ 0.138 & 0.802 $\pm$ 0.120 & 0.558 $\pm$ 0.151 & 0.714 $\pm$ 0.073 & \cellcolor[HTML]{D6F0D6}\textbf{0.968 $\pm$ 0.050} \\
\texttt{iso\_c7} & 0.880 $\pm$ 0.074 & 1.000 $\pm$ 0.000 & 0.880 $\pm$ 0.074 & 0.901 $\pm$ 0.107 & 0.935 $\pm$ 0.081 & 1.000 $\pm$ 0.000 & 1.000 $\pm$ 0.000 \\
\texttt{iso\_c9} & 0.718 $\pm$ 0.110 & 0.808 $\pm$ 0.100 & 0.758 $\pm$ 0.074 & 0.778 $\pm$ 0.104 & 0.843 $\pm$ 0.051 & \textbf{0.855 $\pm$ 0.097} & 0.820 $\pm$ 0.087 \\
\texttt{jnk3} & 0.300 $\pm$ 0.161 & 0.524 $\pm$ 0.211 & 0.298 $\pm$ 0.161 & 0.652 $\pm$ 0.128 & 0.336 $\pm$ 0.248 & 0.536 $\pm$ 0.222 & \cellcolor[HTML]{D6F0D6}\textbf{0.820  $\pm$0.091} \\
\texttt{med\_1} & 0.212 $\pm$ 0.026 & 0.291 $\pm$ 0.023 & 0.228 $\pm$ 0.022 & 0.265 $\pm$ 0.012 & 0.192 $\pm$ 0.020 & \cellcolor[HTML]{D6F0D6}\textbf{0.367 $\pm$ 0.030} & \cellcolor[HTML]{D6F0D6}\textbf{0.359 $\pm$ 0.038} \\
\texttt{med\_2} & 0.211 $\pm$ 0.024 & 0.241 $\pm$ 0.037 & 0.204 $\pm$ 0.028 & 0.241 $\pm$ 0.045 & 0.205 $\pm$ 0.034 & \textbf{0.258 $\pm$ 0.047} & \cellcolor[HTML]{D6F0D6}\textbf{0.297 $\pm$ 0.032} \\
\texttt{mest\_sim} & 0.439 $\pm$ 0.020 & 0.565 $\pm$ 0.068 & 0.428 $\pm$ 0.038 & 0.577 $\pm$ 0.052 & 0.443 $\pm$ 0.028 & \cellcolor[HTML]{D6F0D6}\textbf{0.660 $\pm$ 0.077} & \cellcolor[HTML]{D6F0D6}\textbf{0.844 $\pm$ 0.154} \\
\texttt{osim\_mpo} & 0.786 $\pm$ 0.019 & 0.815 $\pm$ 0.009 & 0.768 $\pm$ 0.015 & 0.803 $\pm$ 0.027 & 0.762 $\pm$ 0.017 & 0.807 $\pm$ 0.016 & \cellcolor[HTML]{D6F0D6}\textbf{0.832 $\pm$ 0.015} \\
\texttt{peri\_mpo} & 0.509 $\pm$ 0.069 & \textbf{0.567 $\pm$ 0.041} & 0.499 $\pm$ 0.044 & 0.504 $\pm$ 0.012 & 0.427 $\pm$ 0.019 & 0.528 $\pm$ 0.010 & 0.555 $\pm$ 0.055 \\
\texttt{rano\_mpo} & 0.763 $\pm$ 0.031 & 0.784 $\pm$ 0.025 & 0.732 $\pm$ 0.028 & 0.748 $\pm$ 0.051 & 0.740 $\pm$ 0.033 & \cellcolor[HTML]{D6F0D6}\textbf{0.799 $\pm$ 0.014} & \textbf{0.788 $\pm$ 0.018} \\
\texttt{logp} & \cellcolor[HTML]{D6F0D6}\textbf{52.790 $\pm$ 5.679} & \cellcolor[HTML]{D6F0D6}\textbf{63.296 $\pm$ 16.073} & \cellcolor[HTML]{D6F0D6}\textbf{53.962 $\pm$ 8.208} & 29.014 $\pm$ 8.537 & 13.554 $\pm$ 1.710 & 21.420 $\pm$ 1.219 & 42.342 $\pm$ 1.350 \\
\texttt{qed} & 0.934 $\pm$ 0.008 & 0.947 $\pm$ 0.002 & 0.934 $\pm$ 0.008 & 0.947 $\pm$ 0.001 & 0.921 $\pm$ 0.008 & 0.941 $\pm$ 0.008 & \cellcolor[HTML]{D6F0D6}\textbf{0.948 $\pm$ 0.000} \\
\texttt{sa} & \textbf{8.952 $\pm$ 0.668} & \textbf{8.947 $\pm$ 0.581} & 8.590 $\pm$ 0.473 & 7.841 $\pm$ 0.310 & 8.066 $\pm$ 0.225 & 8.703 $\pm$ 0.557 & 8.752 $\pm$ 0.083 \\
\texttt{scaf\_hop} & 0.479 $\pm$ 0.027 & 0.506 $\pm$ 0.021 & 0.463 $\pm$ 0.010 & 0.511 $\pm$ 0.020 & 0.459 $\pm$ 0.005 & 0.504 $\pm$ 0.018 & \cellcolor[HTML]{D6F0D6}\textbf{0.557 $\pm$ 0.038} \\
\texttt{sita\_mpo} & 0.310 $\pm$ 0.036 & \textbf{0.512 $\pm$ 0.216} & 0.319 $\pm$ 0.032 & \textbf{0.410 $\pm$ 0.121} & 0.360 $\pm$ 0.047 & 0.394 $\pm$ 0.183 & 0.327 $\pm$ 0.056 \\
\texttt{thio\_red} & 0.327 $\pm$ 0.035 & 0.420 $\pm$ 0.053 & 0.311 $\pm$ 0.022 & 0.400 $\pm$ 0.030 & 0.313 $\pm$ 0.024 & \cellcolor[HTML]{D6F0D6}\textbf{0.489 $\pm$ 0.036} & \cellcolor[HTML]{D6F0D6}\textbf{0.599 $\pm$ 0.107} \\
\texttt{trog\_red} & 0.250 $\pm$ 0.009 & 0.287 $\pm$ 0.028 & 0.237 $\pm$ 0.008 & 0.286 $\pm$ 0.009 & 0.232 $\pm$ 0.014 & \cellcolor[HTML]{D6F0D6}\textbf{0.364 $\pm$ 0.024} & \cellcolor[HTML]{D6F0D6}\textbf{0.443 $\pm$ 0.052} \\
\texttt{vals\_sma} & 0.000 $\pm$ 0.000 & \textbf{0.001 $\pm$ 0.002} & 0.000 $\pm$ 0.000 & \textbf{0.053 $\pm$ 0.118} & 0.000 $\pm$ 0.000 & 0.000 $\pm$ 0.000 & 0.000 $\pm$ 0.000 \\
\texttt{zale\_mpo} & 0.454 $\pm$ 0.083 & 0.455 $\pm$ 0.022 & 0.401 $\pm$ 0.027 & 0.457 $\pm$ 0.027 & 0.391 $\pm$ 0.036 & \textbf{0.472 $\pm$ 0.023} & \cellcolor[HTML]{D6F0D6}\textbf{0.521 $\pm$ 0.025} \\
\bottomrule
\end{tabular}}
\caption{\textbf{Batch regime, uncoupled generation.} Final best score on all 25 PMO objectives after 500 acquisitions (100 initial points, then 100 epochs of 5), mean $\pm$ sample standard deviation over 5 seeds. AFGG generates each batch \emph{uncoupled} here: every molecule is guided by its own acquisition value, and the batch only forms at selection. Every cell is judged against the opposing group: an AFGG cell against the five baselines, a baseline cell against the two AFGG columns. \textbf{Bold}: the cell's mean exceeds every mean in the opposing group. \colorbox[HTML]{D6F0D6}{Green}: that lead is significant, exceeding the cell's own standard deviation; bold without green therefore means ahead but within noise.}

\label{tab:afgg-batch-uncoupled}
\end{table}

\begin{figure}[t]
    \centering
    \includegraphics[width=\textwidth]{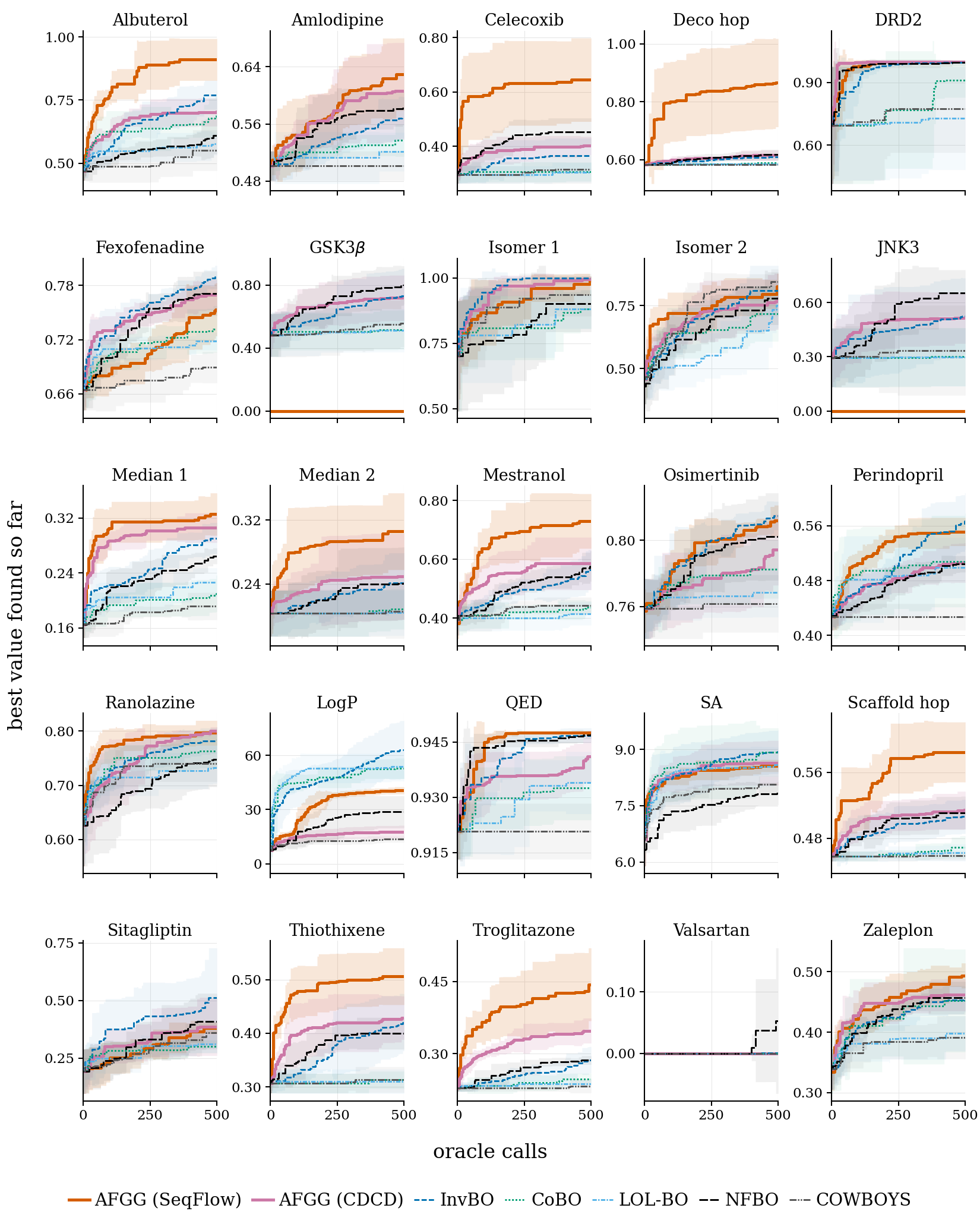}
    \caption{\textbf{Batch regime, coupled generation}: 500 acquisitions (100 initial points, then 100 epochs of 5), all 25 objectives. }
    \label{fig:25-obj-batch-coupled}
\end{figure}

\begin{table}[t]
\centering
\small
\resizebox{\textwidth}{!}{%
\begin{tabular}{lrrrrrrr}
\toprule
 & \multicolumn{5}{c}{baselines} & \multicolumn{2}{c}{AFGG} \\
\cmidrule(lr){2-6}\cmidrule(lr){7-8}
Objective & CoBO & InvBO & LOL-BO & NFBO & COWBOYS & CDCD & SeqFlow \\
\midrule
\texttt{albu\_sim} & 0.692 $\pm$ 0.095 & 0.771 $\pm$ 0.057 & 0.579 $\pm$ 0.043 & 0.609 $\pm$ 0.073 & 0.551 $\pm$ 0.081 & 0.703 $\pm$ 0.057 & \cellcolor[HTML]{D6F0D6}\textbf{0.910 $\pm$ 0.083} \\
\texttt{amlo\_mpo} & 0.540 $\pm$ 0.021 & 0.569 $\pm$ 0.029 & 0.522 $\pm$ 0.026 & 0.582 $\pm$ 0.024 & 0.501 $\pm$ 0.021 & \textbf{0.606 $\pm$ 0.068} & \textbf{0.630 $\pm$ 0.051} \\
\texttt{cele\_red} & 0.307 $\pm$ 0.035 & 0.366 $\pm$ 0.064 & 0.304 $\pm$ 0.036 & 0.452 $\pm$ 0.052 & 0.316 $\pm$ 0.052 & 0.402 $\pm$ 0.034 & \cellcolor[HTML]{D6F0D6}\textbf{0.645 $\pm$ 0.153} \\
\texttt{deco\_hop} & 0.588 $\pm$ 0.005 & 0.610 $\pm$ 0.007 & 0.586 $\pm$ 0.004 & 0.617 $\pm$ 0.018 & 0.583 $\pm$ 0.004 & 0.615 $\pm$ 0.018 & \cellcolor[HTML]{D6F0D6}\textbf{0.866 $\pm$ 0.153} \\
\texttt{drd2} & 0.930 $\pm$ 0.090 & 0.996 $\pm$ 0.006 & 0.728 $\pm$ 0.247 & 0.994 $\pm$ 0.012 & 0.774 $\pm$ 0.210 & \cellcolor[HTML]{D6F0D6}\textbf{1.000 $\pm$ 0.000} & \cellcolor[HTML]{D6F0D6}\textbf{1.000 $\pm$ 0.000} \\
\texttt{fexo\_mpo} & 0.732 $\pm$ 0.025 & \cellcolor[HTML]{D6F0D6}\textbf{0.789 $\pm$ 0.013} & 0.718 $\pm$ 0.010 & \textbf{0.771 $\pm$ 0.025} & 0.689 $\pm$ 0.017 & 0.770 $\pm$ 0.017 & 0.753 $\pm$ 0.029 \\
\texttt{gsk3b} & 0.516 $\pm$ 0.125 & 0.730 $\pm$ 0.174 & 0.512 $\pm$ 0.138 & 0.802 $\pm$ 0.120 & 0.558 $\pm$ 0.151 & 0.718 $\pm$ 0.144 & \cellcolor[HTML]{D6F0D6}\textbf{0.950 $\pm$ 0.044} \\
\texttt{iso\_c7} & 0.880 $\pm$ 0.074 & \cellcolor[HTML]{D6F0D6}\textbf{1.000 $\pm$ 0.000} & 0.880 $\pm$ 0.074 & 0.901 $\pm$ 0.107 & 0.935 $\pm$ 0.081 & 0.995 $\pm$ 0.011 & 0.995 $\pm$ 0.011 \\
\texttt{iso\_c9} & 0.718 $\pm$ 0.110 & 0.808 $\pm$ 0.100 & 0.758 $\pm$ 0.074 & 0.778 $\pm$ 0.104 & \textbf{0.843 $\pm$ 0.051} & 0.777 $\pm$ 0.127 & 0.825 $\pm$ 0.052 \\
\texttt{jnk3} & 0.300 $\pm$ 0.161 & 0.524 $\pm$ 0.211 & 0.298 $\pm$ 0.161 & 0.652 $\pm$ 0.128 & 0.336 $\pm$ 0.248 & 0.514 $\pm$ 0.140 & \textbf{0.704 $\pm$ 0.110} \\
\texttt{med\_1} & 0.212 $\pm$ 0.026 & 0.291 $\pm$ 0.023 & 0.228 $\pm$ 0.022 & 0.265 $\pm$ 0.012 & 0.192 $\pm$ 0.020 & \textbf{0.306 $\pm$ 0.022} & \cellcolor[HTML]{D6F0D6}\textbf{0.326 $\pm$ 0.031} \\
\texttt{med\_2} & 0.211 $\pm$ 0.024 & 0.241 $\pm$ 0.037 & 0.204 $\pm$ 0.028 & 0.241 $\pm$ 0.045 & 0.205 $\pm$ 0.034 & \textbf{0.249 $\pm$ 0.054} & \cellcolor[HTML]{D6F0D6}\textbf{0.306 $\pm$ 0.048} \\
\texttt{mest\_sim} & 0.439 $\pm$ 0.020 & 0.565 $\pm$ 0.068 & 0.428 $\pm$ 0.038 & 0.577 $\pm$ 0.052 & 0.443 $\pm$ 0.028 & \textbf{0.587 $\pm$ 0.089} & \cellcolor[HTML]{D6F0D6}\textbf{0.730 $\pm$ 0.093} \\
\texttt{osim\_mpo} & 0.786 $\pm$ 0.019 & \textbf{0.815 $\pm$ 0.009} & 0.768 $\pm$ 0.015 & 0.803 $\pm$ 0.027 & 0.762 $\pm$ 0.017 & 0.794 $\pm$ 0.021 & 0.812 $\pm$ 0.009 \\
\texttt{peri\_mpo} & 0.509 $\pm$ 0.069 & \textbf{0.567 $\pm$ 0.041} & 0.499 $\pm$ 0.044 & 0.504 $\pm$ 0.012 & 0.427 $\pm$ 0.019 & 0.506 $\pm$ 0.029 & 0.551 $\pm$ 0.021 \\
\texttt{rano\_mpo} & 0.763 $\pm$ 0.031 & 0.784 $\pm$ 0.025 & 0.732 $\pm$ 0.028 & 0.748 $\pm$ 0.051 & 0.740 $\pm$ 0.033 & \cellcolor[HTML]{D6F0D6}\textbf{0.800 $\pm$ 0.007} & \textbf{0.797 $\pm$ 0.023} \\
\texttt{logp} & \cellcolor[HTML]{D6F0D6}\textbf{52.790 $\pm$ 5.679} & \cellcolor[HTML]{D6F0D6}\textbf{63.296 $\pm$ 16.073} & \cellcolor[HTML]{D6F0D6}\textbf{53.962 $\pm$ 8.208} & 29.014 $\pm$ 8.537 & 13.554 $\pm$ 1.710 & 17.799 $\pm$ 3.424 & 40.594 $\pm$ 2.014 \\
\texttt{qed} & 0.934 $\pm$ 0.008 & 0.947 $\pm$ 0.002 & 0.934 $\pm$ 0.008 & 0.947 $\pm$ 0.001 & 0.921 $\pm$ 0.008 & 0.941 $\pm$ 0.007 & \cellcolor[HTML]{D6F0D6}\textbf{0.948 $\pm$ 0.000} \\
\texttt{sa} & \textbf{8.952 $\pm$ 0.668} & \textbf{8.947 $\pm$ 0.581} & 8.590 $\pm$ 0.473 & 7.841 $\pm$ 0.310 & 8.066 $\pm$ 0.225 & 8.645 $\pm$ 0.597 & 8.570 $\pm$ 0.163 \\
\texttt{scaf\_hop} & 0.479 $\pm$ 0.027 & 0.506 $\pm$ 0.021 & 0.463 $\pm$ 0.010 & 0.511 $\pm$ 0.020 & 0.459 $\pm$ 0.005 & \textbf{0.514 $\pm$ 0.023} & \cellcolor[HTML]{D6F0D6}\textbf{0.584 $\pm$ 0.036} \\
\texttt{sita\_mpo} & 0.310 $\pm$ 0.036 & \textbf{0.512 $\pm$ 0.216} & 0.319 $\pm$ 0.032 & \textbf{0.410 $\pm$ 0.121} & 0.360 $\pm$ 0.047 & 0.387 $\pm$ 0.126 & 0.381 $\pm$ 0.025 \\
\texttt{thio\_red} & 0.327 $\pm$ 0.035 & 0.420 $\pm$ 0.053 & 0.311 $\pm$ 0.022 & 0.400 $\pm$ 0.030 & 0.313 $\pm$ 0.024 & \textbf{0.429 $\pm$ 0.040} & \cellcolor[HTML]{D6F0D6}\textbf{0.506 $\pm$ 0.053} \\
\texttt{trog\_red} & 0.250 $\pm$ 0.009 & 0.287 $\pm$ 0.028 & 0.237 $\pm$ 0.008 & 0.286 $\pm$ 0.009 & 0.232 $\pm$ 0.014 & \cellcolor[HTML]{D6F0D6}\textbf{0.348 $\pm$ 0.026} & \cellcolor[HTML]{D6F0D6}\textbf{0.444 $\pm$ 0.076} \\
\texttt{vals\_sma} & 0.000 $\pm$ 0.000 & \textbf{0.001 $\pm$ 0.002} & 0.000 $\pm$ 0.000 & \textbf{0.053 $\pm$ 0.118} & 0.000 $\pm$ 0.000 & 0.000 $\pm$ 0.000 & 0.000 $\pm$ 0.000 \\
\texttt{zale\_mpo} & 0.454 $\pm$ 0.083 & 0.455 $\pm$ 0.022 & 0.401 $\pm$ 0.027 & 0.457 $\pm$ 0.027 & 0.391 $\pm$ 0.036 & \textbf{0.462 $\pm$ 0.027} & \cellcolor[HTML]{D6F0D6}\textbf{0.493 $\pm$ 0.021} \\
\bottomrule
\end{tabular}}
\caption{\textbf{Batch regime, coupled generation.} Final best score on all 25 PMO objectives after 500 acquisitions (100 initial points, then 100 epochs of 5), mean $\pm$ sample standard deviation over 5 seeds. AFGG generates each batch \emph{coupled} here: one joint qLogEI tilt per candidate batch, so its members are generated aware of each other. Every cell is judged against the opposing group: an AFGG cell against the five baselines, a baseline cell against the two AFGG columns. \textbf{Bold}: the cell's mean exceeds every mean in the opposing group. \colorbox[HTML]{D6F0D6}{Green}: that lead is significant, exceeding the cell's own standard deviation; bold without green therefore means ahead but within noise.}
\label{tab:afgg-batch-coupled}
\end{table}

\subsection{Effect of the \rev{Generative Prior}}
\label{sec:generative-prior}

AFGG (CDCD) and AFGG (SeqFlow) share the same acquisition-guided machinery, but differ in the generative model being steered: \tabref{tab:gen-quality} isolates that component: we draw $5000$ unconditional samples from each model and score them against a $30$k-molecule subset of GuacaMol. Validity, uniqueness and novelty do not separate the two models: all six values exceed $0.96$ and the residual differences, below two points, are in CDCD's \rev{favor}. FCD does separate them, by a factor of roughly $3.5$: SeqFlow reaches $6.02$ against CDCD's $20.99$, i.e. its samples sit far closer to the reference distribution in ChemNet feature space. Guidance tilts the prior rather than extending its support, so what the acquisition function can propose is confined to the region the prior already covers, and a prior whose mass is better aligned with the data distribution gives the search a better-conditioned region to act on.

\begin{table}[t]
\centering
\begin{tabular}{lcc}
\toprule
Metric & CDCD & SeqFlow \\
\midrule
Validity $\uparrow$   & \textbf{1.0000}             & 0.9972 \\
Uniqueness $\uparrow$ & \textbf{0.9932}             & 0.9862 \\
Novelty $\uparrow$    & \textbf{0.9956}             & 0.9838 \\
FCD $\downarrow$      & 20.99                       & \textbf{6.02} \\
\bottomrule
\end{tabular}
\caption{Unconditional sample quality. $5000$ molecules drawn from each generative model and scored against a $30$k-molecule subset of GuacaMol.}
\label{tab:gen-quality}
\end{table}

\subsection{Hyperparameter \rev{Tuning for} CDCD}
\label{sec:hyperparameter-tuning}

Every column reports the final best score after 300 oracle calls, mean $\pm$ sample standard deviation over 5 seeds. \rev{All ablations use the CDCD backbone in the sequential setting} (10 initial points, 300 oracle calls, 1 point per epoch). Every objective is scored on $[0,1]$; bold marks the best mean in each row (ties at the printed precision are all bolded), and \emph{Wins} counts only rows with a unique winner. 

\paragraph{\rev{Guidance on vs.\ off.}}
By setting $\lambda=0$, it has been proven that the generative model is not the main driver for optimizing the function. Having an additional force, represented by the \rev{acquisition} function\rev{,} has a \rev{dramatic} effect \rev{on performance (see} \tabref{tab:ablation-guidance} for full results\rev{).}
\begin{table}[t]
\centering
\small
\begin{tabular}{lrr}
\toprule
 & \multicolumn{2}{c}{Guidance on vs off} \\
\cmidrule(lr){2-3}
Objective & on & off ($\lambda{=}0$) \\
\midrule
\texttt{osim\_mpo} & 0.714 $\pm$ 0.064 & \textbf{0.722 $\pm$ 0.048} \\
\texttt{med\_2} & \textbf{0.199 $\pm$ 0.013} & 0.173 $\pm$ 0.012 \\
\texttt{amlo\_mpo} & \textbf{0.486 $\pm$ 0.053} & 0.408 $\pm$ 0.018 \\
\texttt{zale\_mpo} & 0.420 $\pm$ 0.054 & \textbf{0.431 $\pm$ 0.030} \\
\texttt{rano\_mpo} & \textbf{0.788 $\pm$ 0.012} & 0.745 $\pm$ 0.026 \\
\texttt{peri\_mpo} & \textbf{0.443 $\pm$ 0.018} & 0.405 $\pm$ 0.073 \\
\texttt{trog\_red} & \textbf{0.308 $\pm$ 0.020} & 0.238 $\pm$ 0.016 \\
\midrule
MEAN & 0.480 & 0.446 \\
Wins & 5 & 2 \\
\bottomrule
\end{tabular}
\caption{Guidance on vs.\ off; \emph{off} sets $\lambda=0$, flattening the SNIS weights to uniform so the tilted mean equals the untilted one. }
\label{tab:ablation-guidance}
\end{table}

\paragraph{Guidance estimator.} SNIS and REINFORCE target the same quantity, $\ztilt$ (\propref{prop:reinforce} and \propref{prop:snis}). SNIS forms it directly as a weighted average of the sampled embeddings, $\sum_n h_n E(w^{(n)})$: a closed-form combination with no gradient involved \eqref{eq:snis}. REINFORCE instead routes through \eqref{eq:reinforce}, a score-function estimator — these are well known to carry substantially higher variance than a direct, closed-form average for a fixed sample budget, and at $N=10$ trajectories per epoch that variance shows: SNIS wins all seven objectives, several by a wide margin (\texttt{zale\_mpo}: 0.420 vs.\ 0.250), \tabref{tab:ablation-estimator} for full results.

\begin{table}[t]
\centering
\small
\begin{tabular}{lrr}
\toprule
 & \multicolumn{2}{c}{Guidance estimator} \\
\cmidrule(lr){2-3}
Objective & SNIS (closed form) & REINFORCE \\
\midrule
\texttt{osim\_mpo} & \textbf{0.714 $\pm$ 0.064} & 0.671 $\pm$ 0.033 \\
\texttt{med\_2} & \textbf{0.199 $\pm$ 0.013} & 0.170 $\pm$ 0.010 \\
\texttt{amlo\_mpo} & \textbf{0.486 $\pm$ 0.053} & 0.410 $\pm$ 0.062 \\
\texttt{zale\_mpo} & \textbf{0.420 $\pm$ 0.054} & 0.250 $\pm$ 0.085 \\
\texttt{rano\_mpo} & \textbf{0.788 $\pm$ 0.012} & 0.687 $\pm$ 0.036 \\
\texttt{peri\_mpo} & \textbf{0.443 $\pm$ 0.018} & 0.382 $\pm$ 0.055 \\
\texttt{trog\_red} & \textbf{0.308 $\pm$ 0.020} & 0.294 $\pm$ 0.013 \\
\midrule
MEAN & 0.480 & 0.409 \\
Wins & 7 & 0 \\
\bottomrule
\end{tabular}
\caption{Guidance estimator: closed-form SNIS vs.\ REINFORCE (score-function estimator with a leave-one-out baseline) in place of the mean-swap identity.}
\label{tab:ablation-estimator}
\end{table}

\paragraph{Acquisition \rev{function.}} \tabref{tab:ablation-acquisition} shows \rev{AFGG's versatility: it can be} applied with any \rev{acquisition} function. \rev{As expected}, EI yields higher \rev{performance} than PI or UCB.

\begin{table}[t]
\centering
\small
\begin{tabular}{lrrr}
\toprule
 & \multicolumn{3}{c}{Acquisition function} \\
\cmidrule(lr){2-4}
Objective & EI & PI & UCB \\
\midrule
\texttt{osim\_mpo} & \textbf{0.714 $\pm$ 0.064} & 0.708 $\pm$ 0.065 & 0.696 $\pm$ 0.041 \\
\texttt{med\_2} & \textbf{0.199 $\pm$ 0.013} & 0.183 $\pm$ 0.015 & 0.170 $\pm$ 0.009 \\
\texttt{amlo\_mpo} & 0.486 $\pm$ 0.053 & \textbf{0.495 $\pm$ 0.029} & 0.490 $\pm$ 0.013 \\
\texttt{zale\_mpo} & 0.420 $\pm$ 0.054 & 0.439 $\pm$ 0.045 & \textbf{0.445 $\pm$ 0.055} \\
\texttt{rano\_mpo} & \textbf{0.788 $\pm$ 0.012} & 0.753 $\pm$ 0.026 & 0.746 $\pm$ 0.037 \\
\texttt{peri\_mpo} & 0.443 $\pm$ 0.018 & \textbf{0.469 $\pm$ 0.053} & 0.391 $\pm$ 0.044 \\
\texttt{trog\_red} & \textbf{0.308 $\pm$ 0.020} & 0.300 $\pm$ 0.016 & 0.235 $\pm$ 0.015 \\
\midrule
MEAN & 0.480 & 0.478 & 0.453 \\
Wins & 4 & 2 & 1 \\
\bottomrule
\end{tabular}
\caption{Acquisition function driving both the guidance tilt and candidate selection; UCB uses $\beta=0.1$\rev{.}}
\label{tab:ablation-acquisition}
\end{table}

\paragraph{Candidate pool size.} Moving from 1 to 10 trajectories captures essentially all the achievable gain (mean $0.470\to0.480$); doubling again to 20 does not clearly help (mean $0.477$, within a single row's typical standard deviation of the 10-trajectory result), while doubling the per-epoch generation cost. We fix 10 as the point where that curve has flattened, \tabref{tab:ablation-pool} for full results.

\begin{table}[t]
\centering
\small
\begin{tabular}{lrrrr}
\toprule
 & \multicolumn{4}{c}{Candidate pool size} \\
\cmidrule(lr){2-5}
Objective & 1 & 5 & 10 & 20 \\
\midrule
\texttt{osim\_mpo} & \textbf{0.736 $\pm$ 0.010} & 0.712 $\pm$ 0.053 & 0.714 $\pm$ 0.064 & 0.711 $\pm$ 0.063 \\
\texttt{med\_2} & 0.186 $\pm$ 0.012 & 0.192 $\pm$ 0.008 & 0.199 $\pm$ 0.013 & \textbf{0.204 $\pm$ 0.007} \\
\texttt{amlo\_mpo} & 0.486 $\pm$ 0.045 & \textbf{0.504 $\pm$ 0.038} & 0.486 $\pm$ 0.053 & 0.496 $\pm$ 0.029 \\
\texttt{zale\_mpo} & \textbf{0.432 $\pm$ 0.048} & 0.411 $\pm$ 0.084 & 0.420 $\pm$ 0.054 & 0.406 $\pm$ 0.060 \\
\texttt{rano\_mpo} & 0.764 $\pm$ 0.008 & 0.768 $\pm$ 0.019 & \textbf{0.788 $\pm$ 0.012} & 0.771 $\pm$ 0.015 \\
\texttt{peri\_mpo} & 0.407 $\pm$ 0.045 & 0.448 $\pm$ 0.047 & 0.443 $\pm$ 0.018 & \textbf{0.453 $\pm$ 0.031} \\
\texttt{trog\_red} & 0.274 $\pm$ 0.018 & 0.288 $\pm$ 0.015 & \textbf{0.308 $\pm$ 0.020} & 0.294 $\pm$ 0.011 \\
\midrule
MEAN & 0.470 & 0.475 & 0.480 & 0.477 \\
Wins & 2 & 1 & 2 & 2 \\
\bottomrule
\end{tabular}
\caption{Candidate pool size: guided trajectories rolled out per epoch, of which the best by acquisition value is evaluated.}
\label{tab:ablation-pool}
\end{table}

\paragraph{Where the temperature acts.} Power tilt reshapes the SNIS weights, $h_n\propto\ell_n^{1/\gamma}$ \rev{(see} \eqref{eq:power_tilt}\rev{)} — always a convex combination of the sampled embeddings, so the guided target cannot leave their span. Scale drift instead rescales the drift after the tilted mean is formed \rev{(see} \eqref{eq:linear_tilt}\rev{)}: at $\lambda>1$ this extrapolates past $\rev{\ztilt}$, pushing the implied target outside the convex hull of what was actually sampled and decoded — a region the frozen denoiser never saw in training. On this seven-objective slice the two mechanisms land close together (means $0.474$–$0.486$, all within a row's typical standard deviation), so we are not claiming a large empirical gap here; the convex-hull argument, not this table, is why we default to power tilt at higher temperatures; \tabref{tab:ablation-temperature} for full results.

\begin{table}[t]
\centering
\small
\begin{tabular}{lrrrr}
\toprule
 & \multicolumn{4}{c}{Where the temperature acts} \\
\cmidrule(lr){2-5}
Objective & power tilt $\gamma{=}2$ & scale drift $\lambda{=}0.5$ & scale drift $\lambda{=}2$ & scale drift $\lambda{=}5$ \\
\midrule
\texttt{osim\_mpo} & \textbf{0.714 $\pm$ 0.064} & 0.702 $\pm$ 0.034 & 0.694 $\pm$ 0.040 & 0.682 $\pm$ 0.026 \\
\texttt{med\_2} & \textbf{0.199 $\pm$ 0.013} & 0.194 $\pm$ 0.010 & 0.191 $\pm$ 0.013 & 0.182 $\pm$ 0.007 \\
\texttt{amlo\_mpo} & 0.486 $\pm$ 0.053 & 0.474 $\pm$ 0.029 & 0.423 $\pm$ 0.029 & \textbf{0.496 $\pm$ 0.051} \\
\texttt{zale\_mpo} & 0.420 $\pm$ 0.054 & \textbf{0.449 $\pm$ 0.026} & 0.402 $\pm$ 0.041 & 0.401 $\pm$ 0.029 \\
\texttt{rano\_mpo} & \textbf{0.788 $\pm$ 0.012} & 0.750 $\pm$ 0.031 & 0.782 $\pm$ 0.023 & 0.773 $\pm$ 0.023 \\
\texttt{peri\_mpo} & 0.443 $\pm$ 0.018 & \textbf{0.458 $\pm$ 0.043} & 0.455 $\pm$ 0.056 & 0.407 $\pm$ 0.029 \\
\texttt{trog\_red} & \textbf{0.308 $\pm$ 0.020} & 0.278 $\pm$ 0.014 & 0.300 $\pm$ 0.022 & 0.302 $\pm$ 0.027 \\
\midrule
MEAN & 0.480 & 0.472 & 0.464 & 0.463 \\
Wins & 4 & 2 & 0 & 1 \\
\bottomrule
\end{tabular}
\caption{Different temperature parameters and how they affect the final generation process\rev{.}}
\label{tab:ablation-temperature}
\end{table}

\end{document}